\documentclass{article}

\usepackage{microtype}
\usepackage{graphicx}
\usepackage{subfigure}
\usepackage{booktabs} 

\usepackage{hyperref}

\usepackage{multirow}
\usepackage{url}

\usepackage{tikz}
\usetikzlibrary{arrows,shapes,backgrounds,fit,positioning}
\usetikzlibrary{shapes.geometric}
\usetikzlibrary{calc}
\usetikzlibrary{arrows.meta}

\usepackage{amsmath, amsthm, amssymb, amsbsy, mathtools}

\newcommand{\cut}[1]{}

\newcommand{\overridecommand}[2]{
  \providecommand{#1}{}
  \renewcommand{#1}{#2}
}

\overridecommand{\aa}{\mathbf{a}}
\overridecommand{\bb}{\mathbf{b}}
\overridecommand{\cc}{\mathbf{c}}
\overridecommand{\dd}{\mathbf{d}}
\overridecommand{\ee}{\mathbf{e}}
\overridecommand{\ff}{\mathbf{f}}
\overridecommand{\bmg}{\mathbf{g}}   
\overridecommand{\hh}{\mathbf{h}}
\overridecommand{\ii}{\mathbf{i}}
\overridecommand{\jj}{\mathbf{j}}
\overridecommand{\kk}{\mathbf{k}}
\overridecommand{\bml}{\mathbf{l}}   
\overridecommand{\mm}{\mathbf{m}}
\overridecommand{\nn}{\mathbf{n}}
\overridecommand{\oo}{\mathbf{o}}
\overridecommand{\pp}{\mathbf{p}}
\overridecommand{\qq}{\mathbf{q}}
\overridecommand{\rr}{\mathbf{r}}
\overridecommand{\ss}{\mathbf{s}}
\overridecommand{\tt}{\mathbf{t}}
\overridecommand{\uu}{\mathbf{u}}
\overridecommand{\vv}{\mathbf{v}}
\overridecommand{\ww}{\mathbf{w}}
\overridecommand{\xx}{\mathbf{x}}
\overridecommand{\yy}{\mathbf{y}}
\overridecommand{\zz}{\mathbf{z}}

\overridecommand{\BB}{\mathbf{B}}
\overridecommand{\CC}{\mathbf{C}}
\overridecommand{\DD}{\mathbf{D}}
\overridecommand{\EE}{\mathbf{E}}
\overridecommand{\FF}{\mathbf{F}}
\overridecommand{\GG}{\mathbf{G}}
\overridecommand{\HH}{\mathbf{H}}
\overridecommand{\II}{\mathbf{I}}
\overridecommand{\JJ}{\mathbf{J}}
\overridecommand{\KK}{\mathbf{K}}
\overridecommand{\LL}{\mathbf{L}}
\overridecommand{\MM}{\mathbf{M}}
\overridecommand{\NN}{\mathbf{N}}
\overridecommand{\OO}{\mathbf{O}}
\overridecommand{\PP}{\mathbf{P}}
\overridecommand{\QQ}{\mathbf{Q}}
\overridecommand{\RR}{\mathbf{R}}
\overridecommand{\SS}{\mathbf{S}}
\overridecommand{\TT}{\mathbf{T}}
\overridecommand{\UU}{\mathbf{U}}
\overridecommand{\VV}{\mathbf{V}}
\overridecommand{\WW}{\mathbf{W}}
\overridecommand{\XX}{\mathbf{X}}
\overridecommand{\YY}{\mathbf{Y}}
\overridecommand{\ZZ}{\mathbf{Z}}

\overridecommand{\aalpha}{\boldsymbol{\alpha}}
\overridecommand{\bbeta}{\boldsymbol{\beta}}
\overridecommand{\ggamma}{\boldsymbol{\gamma}}
\overridecommand{\ddelta}{\boldsymbol{\delta}}
\overridecommand{\eepsilon}{\boldsymbol{\epsilon}}
\overridecommand{\vvarepsilon}{\boldsymbol{\varepsilon}}
\overridecommand{\zzeta}{\boldsymbol{\zeta}}
\overridecommand{\eeta}{\boldsymbol{\eta}}
\overridecommand{\ttheta}{\boldsymbol{\theta}}
\overridecommand{\vvartheta}{\boldsymbol{\vartheta}}
\overridecommand{\iiota}{\boldsymbol{\iota}}
\overridecommand{\kkappa}{\boldsymbol{\kappa}}
\overridecommand{\llambda}{\boldsymbol{\lambda}}
\overridecommand{\mmu}{\boldsymbol{\mu}}
\overridecommand{\nnu}{\boldsymbol{\nu}}
\overridecommand{\xxi}{\boldsymbol{\xi}}
\overridecommand{\ppi}{\boldsymbol{\pi}}
\overridecommand{\vvarpi}{\boldsymbol{\varpi}}
\overridecommand{\rrho}{\boldsymbol{\rho}}
\overridecommand{\vvarrho}{\boldsymbol{\varrho}}
\overridecommand{\ssigma}{\boldsymbol{\sigma}}
\overridecommand{\vvarsigma}{\boldsymbol{\varsigma}}
\overridecommand{\ttau}{\boldsymbol{\tau}}
\overridecommand{\uupsilon}{\boldsymbol{\upsilon}}
\overridecommand{\pphi}{\boldsymbol{\phi}}
\overridecommand{\vvarphi}{\boldsymbol{\varphi}}
\overridecommand{\cchi}{\boldsymbol{\chi}}
\overridecommand{\ppsi}{\boldsymbol{\psi}}
\overridecommand{\oomega}{\boldsymbol{\omega}}
\overridecommand{\GGamma}{\boldsymbol{\Gamma}}
\overridecommand{\DDelta}{\boldsymbol{\Delta}}
\overridecommand{\TTheta}{\boldsymbol{\Theta}}
\overridecommand{\LLambda}{\boldsymbol{\Lambda}}
\overridecommand{\XXi}{\boldsymbol{\Xi}}
\overridecommand{\PPi}{\boldsymbol{\Pi}}
\overridecommand{\SSigma}{\boldsymbol{\Sigma}}
\overridecommand{\UUpsilon}{\boldsymbol{\Upsilon}}
\overridecommand{\PPhi}{\boldsymbol{\Phi}}
\overridecommand{\PPsi}{\boldsymbol{\Psi}}
\overridecommand{\OOmega}{\boldsymbol{\Omega}}

\newcommand{\D}{\mathcal{D}}  
\newcommand{\E}{\mathbb{E}}   
\newcommand{\F}{\mathcal{F}}
\newcommand{\I}{\mathcal{I}}

\newcommand{\N}{\mathcal{N}}

\newcommand{\R}{\mathbb{R}}   

\newcommand{\1}{\mathbf{1}}   

\DeclareMathOperator*{\argmax}{argmax}

\newtheorem{proposition}{Proposition}

\newtheorem{definition}{Definition}

\usepackage[accepted]{icml2020}

\usepackage{inconsolata}

\icmltitlerunning{
Learning from Irregularly-Sampled Time Series: A Missing Data Perspective}

\begin{document}

\twocolumn[
\icmltitle{
Learning from Irregularly-Sampled Time Series: A Missing Data Perspective}



\icmlsetsymbol{equal}{*}

\begin{icmlauthorlist}
\icmlauthor{Steven Cheng-Xian Li}{umass}
\icmlauthor{Benjamin M. Marlin}{umass}
\end{icmlauthorlist}

\icmlaffiliation{umass}{University of Massachusetts Amherst}

\icmlcorrespondingauthor{Steven Cheng-Xian Li}{li.stevecx@gmail.com}

\icmlkeywords{Generative Models, Missing Data, Time Series}

\vskip 0.3in
]



\printAffiliationsAndNotice{}  

\begin{abstract}
  Irregularly-sampled time series occur in many domains including healthcare.
They can be challenging to model because they do not naturally yield
a fixed-dimensional representation as required by many
standard machine learning models.
In this paper, we consider irregular sampling from the perspective of
missing data.
We model observed irregularly-sampled time series data as
a sequence of index-value pairs sampled from a continuous
but unobserved function.
We introduce an encoder-decoder framework for learning from
such generic indexed sequences.
We propose learning methods for this framework based on
variational autoencoders and generative adversarial networks.
For continuous irregularly-sampled time series,
we introduce continuous convolutional layers that can efficiently
interface with existing neural network architectures.
Experiments show that our models are able to achieve
competitive or better classification results
on irregularly-sampled multivariate time series
compared to recent RNN models
while offering significantly faster training times.

\end{abstract}

\section{Introduction}

Irregularly-sampled time series are characterized by non-uniform
time intervals between successive measurements.
Such data naturally occur in many real world domains.
For example, in clinical data, an individual patient's state of health
may be recorded only at irregular time intervals with different subsets of
variables observed at different times.
Further, different individuals typically have different numbers of observations
for different subsets of variables observed at different time points,
including after aligning to events like time of admission or disease onset.

These characteristics of irregularly-sampled time series data
create multiple challenges for classical machine learning models
and algorithms that require data to be defined with respect to
a fixed dimensional feature space.
However, there has been significant recent progress on this problem.
For example,
the GRU-D model was proposed
as a direct extension of discrete time RNNs to the case of
continuous time observations \citep{che2018recurrent}.
The model uses exponential decay dynamics applied to either visible
or latent states.
\citet{rubanova2019latent} proposed latent ordinary differential equation (ODE)
models as a more natural way to model continuous dynamics.
Latent ODEs extend the neural ODE model \citep{chen2018neural},
which enables modeling of complex ODEs using neural networks.
However, many of these models can be slow to learn
due to their sequential nature.

The focus of this paper is on learning from a collection of
irregularly-sampled time series that are observed over a fixed time span.
The specific tasks we want to accomplish are:
i) learning the distribution of the latent temporal process,
ii) given a time series, inferring the distribution of
the corresponding latent process, and
iii) classification of time series.
If we view each time series as observations
sampled from a complete latent process defined over a time span $[0,T]$,
this is essentially a missing data problem
as we only have information about the latent process
at a subset of points in time within $[0, T]$.

Learning complex distributions in the presence of missing data
is a problem that has received substantial recent attention.
For example, models have recently been proposed based on
variational autoencoders (VAEs) \citep{kingma2013auto}
such as partial VAEs \citep{ma2018partial,ma2019eddi}
and MIWAE \citep{mattei2019miwae}.
Implicit models based on generative adversarial networks (GANs)
\citep{goodfellow2014generative} have also been recently proposed
such as MisGAN \citep{li2018learning}.
However, these models only work for problems with finite dimensional data
such as recommendation systems or image modeling.
Neural processes \citep{garnelo2018conditional,garnelo2018neural}
can be seen as an extension of partial VAEs
for the continuous space that model distributions over functions.

The main contribution of this paper is the development of a
scalable framework for learning distributions from
irregularly-sampled time series.
We transform modeling such time series data into
a general missing data problem and introduce an encoder-decoder
framework that unifies a number of previous approaches
to modeling incomplete data based on variational autoencoders.
In addition, we propose a GAN-based model for training this framework
that we show outperforms the recently proposed MisGAN model.
We then introduce continuous convolutional layers
for handling irregularly-sampled time series
to efficiently interface with existing neural network architectures.
Experiments show that our framework is able to achieve
competitive or better classification results
on irregularly-sampled multivariate time series classification tasks
compared to recent time series models such as Latent ODE,
while can be trained faster by an order of magnitude.

Our implementation is available at \\
{\small\url{https://github.com/steveli/partial-encoder-decoder}}.

\section{Index Representation for Incomplete Data}
\label{sec:framework}

Suppose we have data defined over an index set $\I$.
We can represent a complete data case as a function $f: \I\to\R$
such that the value of the element associated with an index $t\in\I$ is $f(t)$.
We use $\R^\I$ to denote the space of complete data.
For example, for images of size $h \times w$,
an element of the index set $t\in\I$
corresponds to the coordinates of a pixel and $f(t)$ is the
corresponding pixel value.
The index set $\I$ in this case is the collection of
all possible coordinates, $\{1,\dots,h\}\times\{1,\dots,w\}$.
For time series defined within an interval $[0, T]$,
an index is a timestamp of an observation and
the index set $\I$ is the continuous interval $[0, T]$.

In the incomplete data setting such as time series,
we do not observe the entire $f$.
Instead we have access to a set of values $\xx$ of $f$
associated with a set of indices $\tt$ that is a subset of $\I$.
Following \citet{little2014statistical},
the generative process for an incomplete data case $(\xx,\tt)$
in a dataset $\D=\{(\xx_i,\tt_i)\}_{i=1}^n$ can be
decomposed into three steps:
i) sampling a complete data $f$ from a distribution $p_{\F}(f)$
over $\R^\I$,
ii) sampling a set of indices $\tt=[t_{i}]_{i=1}^{|\tt|}$ 
from a distribution $p_\I(\tt|f)$ over the power set $2^\I$
conditioned on the sampled $f$,
and iii) retaining the values of $f$ at the sampled indices $\tt$
to form a set of
corresponding observed values $\xx = [f(t_{i})]_{i=1}^{|\tt|}$.


We note that this representation of incomplete data is
\emph{permutation invariant}, that is,
the incomplete data $(\xx,\tt)$ is equivalent to
$([x_{\pi(i)}]_{i=1}^{|\tt|},[t_{\pi(i)}]_{i=1}^{|\tt|})$
for any permutation $\pi$ of $\{1,\dots,|\tt|\}$.
We will later discuss why this property is important for constructing the
encoder in Section~\ref{sec:ae}.


The goal of this work is to model the complete data distribution
$p_{\F}$ given only the incomplete observations
contained in the dataset $\D$.
We do not focus on learning the distribution $p_\I$
as this distribution is typically not the primary concern
in the applications we focus on.
For simplicity, we make the further assumption that $f$ and $\tt$
are independent, that is, the generative process of
an incomplete case $(\xx,\tt)$ is given by
\begin{equation}
  f\sim p_{\F}(f), \quad
  \tt \sim p_\I(\tt), \quad
  \xx = [f(t_{i})]_{i=1}^{|\tt|}.
  \label{eq:indep gen}
\end{equation}
In Appendix~\ref{sec:independence}
we will discuss the implications of this assumption and how to relax it.
In the next section, we present models for finite index sets.
In Section~\ref{sec:cont time} we present models for continuous index sets.

\section{Incomplete Data with Finite Index Set}

In this section, we focus on the case where the index set $\I$ is finite.
We begin by describing a base encoder-decoder framework,
which can be trained by models based on VAEs and GANs.

\subsection{Encoder-Decoder Framework}
\label{sec:ae}

We employ a general encoder-decoder framework
for modeling incomplete data.
For the decoder, we model the distribution of the complete data $p_{\F}(f)$
as a two-step procedure:
\begin{equation}
  \zz\sim p_z(\zz),\quad
  f = g_\theta(\zz)
  \label{eq:gen complete}
\end{equation}
where we first draw a latent code $\zz$
from a simple distribution $p_z(\zz)$ such as a standard Gaussian.
We then transform $\zz$ into a complete sample $f\in\R^\I$
through a deterministic function $g_\theta(\zz)$.

The encoder, denoted $q_\phi(\zz|\xx,\tt)$,
aims to model the posterior distribution of the latent code
associated with an incomplete example $(\xx,\tt)$.
Since the representation of incomplete data is permutation invariant
as noted earlier,
the encoder should also be permutation invariant \citep{zaheer2017deep}.
Below we define such a function $m(\xx,\tt)$ that provides
a simple construction of the encoder.
\begin{definition}
  \label{def:masking}
  The masking function $m(\xx,\tt)$ maps an incomplete
  data case $(\xx,\tt)$ to a masked form in $\R^{\I}$
  with all missing entries replaced by zero.
  Specifically, let $\vv=m(\xx,\tt)$ then each entry of $\vv$ has the form
  $v_t = \sum_{i=1}^{|\tt|} x_i\1\{t_i=t\}$ for all $t\in\I$.
\end{definition}

The masking function serves as an interface that transforms
an incomplete data case $(\xx,\tt)$ with arbitrary size
to the masked form $m(\xx,\tt)$ of fixed dimension in $\R^\I$.

We can construct the encoder distribution to have the form of
$q_\phi(\zz|m(\xx,\tt))$,
where the distribution is only parameterized by the fixed-dimensional
masked data $m(\xx,\tt)$.
For example, we can use a Gaussian encoder,
$q_\phi(\zz|\xx,\tt) = \N(\zz | \mu_\phi(\vv), \Sigma_\phi(\vv))$
where $\vv=m(\xx,\tt)$,
with its mean $\mu_\phi$ and diagonal covariance $\Sigma_\phi$
constructed using neural networks.

\begin{figure}[t]
  \centering
  \hspace*{-2em}
  \includegraphics[width=3.2in, trim=0 10em 0 3em, clip]
  {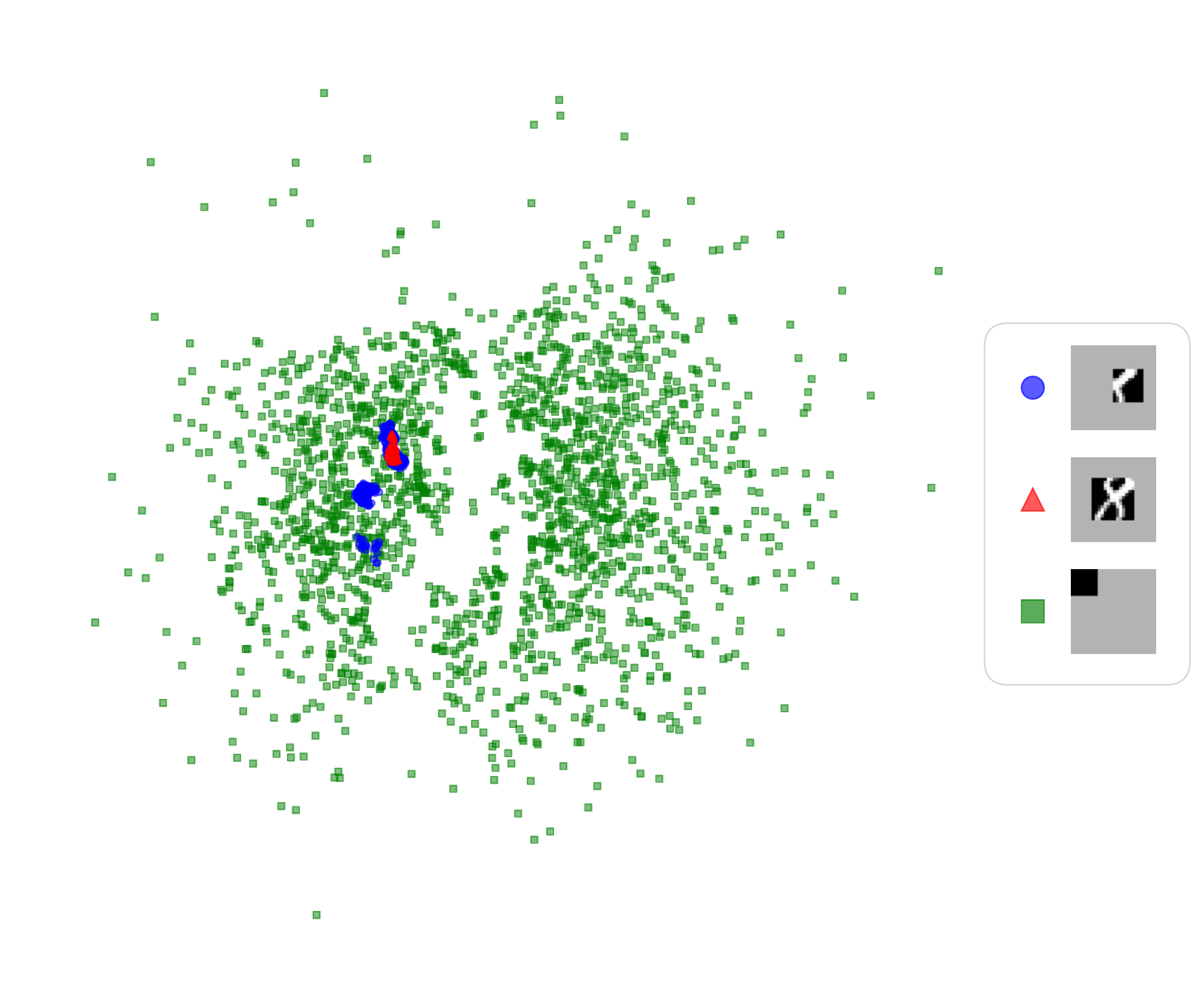} \\
  \includegraphics[width=3.2in]{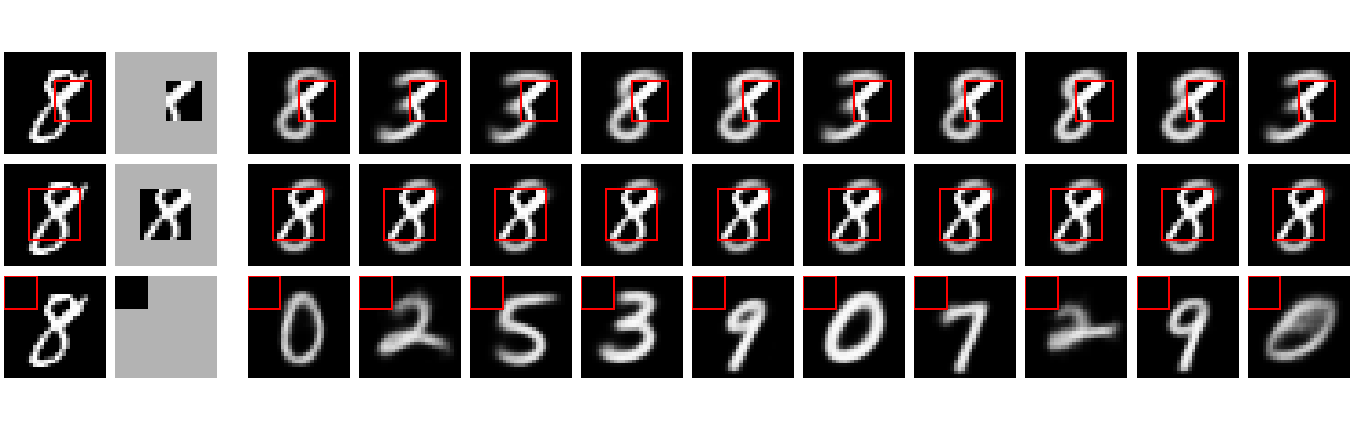}
  \vspace{-1em}
  \caption{
    At the top we plot the 2D latent codes drawn from the
    encoder $q_\phi(\zz|\xx,\tt)$ with three different incomplete
    MNIST examples.
    At the bottom, each row corresponds to one of the three examples
    we encode.
    All three cases come from the same MNIST image as the
    leftmost image in each row except we observe different rectangular
    regions on the image.
    The resulting incomplete images are shown as the second image
    in each row where the gray pixels corresponds to the missing entries.
    The ten images on the right are decoded from the random samples
    drawn from the encoder $q_\phi(\zz|\xx,\tt)$.
    The pixels inside of each red box on those sampled images are the
    observed pixels and those outside are generated by the model (P-VAE)
    described in Section~\ref{sec:impute fid}.
    Note that the blurry completion is due to the insufficient
    capacity of 2D latent codes.
    The latent space plot shows that the second case (red)
    has relatively low uncertainty.
    In contrast, the third case (green) has high uncertainty
    whose encoded distribution looks similar to the Gaussian prior $p_z(\zz)$.
  }
  \label{fig:encode2d}
  \vspace{-1em}
\end{figure}

Note that in the presence of missing data,
we cannot use a deterministic encoder as in standard
autoencoders for complete data,
because different incomplete samples may carry very different levels
of uncertainty as shown in Figure~\ref{fig:encode2d}.
In other words,
there could be many different latent codes $\zz$ that
can be decoded into a variety of complete samples that are
consistent with the observed part of the data.
%

We next describe two training strategies for learning the encoder and decoder.

\subsection{Partial Variational Autoencoder}
\label{sec:pvae}

To train the framework using maximum likelihood,
we construct a proper density model by
adding independent noise to each component of $g_\theta(\zz,t_i)$
for all $t_i\in\tt$,
where $g_\theta(\zz,t_i)$ denotes $f(t_i)$ with $f=g_\theta(\zz)$.
For example, for real-valued data,
the distribution $p(x_i|g_\theta(\zz,t_i))$,
or referred in short as $p_\theta(x_i|\zz,t_i)$, could
be a Gaussian $\N(x_i|f(t_i),\sigma^2)$ with a pre-defined variance $\sigma^2$.
As a result, the joint distribution of an incomplete data case $(\xx,\tt)$ is
\begin{align*}
  p(\xx,\tt)
  &= \int p(\zz) p_\I(\tt) \prod_{i=1}^{|\tt|} p_\theta(x_i|\zz,t_i) d\zz.
\end{align*}
Since this marginal is intractable, we
instead maximize a variational lower bound on $\log p(\xx,\tt)$ given by
\begin{equation}
  \int q_\phi(\zz|\xx,\tt)\log
  \frac{p_z(\zz) p_\I(\tt) \prod_{i=1}^{|\tt|} p_\theta(x_i|\zz,t_i)}
  {q_\phi(\zz|\xx,\tt)} d\zz.
  \label{eq:elbo}
\end{equation}

To learn the distribution of the data parameterized by $g_\theta(\zz)$,
we only need to learn the parameters of $p_\theta(x|\zz,t)$ and
$q_\phi(\zz|\xx,\tt)$, denoted by $\theta$ and $\phi$ respectively.
Due to the assumed independence between $\tt$ and $\zz$,
when taking the derivative of \eqref{eq:elbo}
with respect to $\theta$ and $\phi$,
the term $p_\I(\tt)$ can be dropped.
As a result, the model can be equivalently learned by maximizing
the variational lower bound on the conditional log-likelihood
given below
where $p_\D$ denotes the empirical distribution of the training dataset $\D$:
\begin{equation}
  \E_{(\xx,\tt)\sim p_\D}
  \E_{\zz\sim q_\phi(\zz|\xx,\tt)}\!\left[
    \log \frac{p_z(\zz) \prod_{i=1}^{|\tt|} p_\theta(x_i|\zz,t_i)}
  {q_\phi(\zz|\xx,\tt)} \right]\!\!.\!\!
  \label{eq:pvae}
\end{equation}
This training objective has been previously introduced as
the Partial Variational Autoencoder \citep{ma2018partial,ma2019eddi},
which we abbreviate as P-VAE.
Neural processes \citep{garnelo2018conditional,garnelo2018neural}
and MIWAE \citep{mattei2019miwae} also have the similar structure.
All of these previous approaches are introduced as optimizing
a conditional objective directly while
here we start with the complete generative process that takes the
point process $p_\I$ into account.
See Appendix~\ref{sec:independence}
on the general setting without the independence assumption.

Similar to VAEs, we can use reparameterizable distributions for
the encoder $q_\phi(\zz|\xx,\tt)$, such as Gaussians as we described in
Section~\ref{sec:ae}.
There are various techniques to construct more expressive encoders
that can also be used in our case.
For example,
we can apply inverse autoregressive flows \citep{kingma2016improved}
to transform distributions or
use semi-implicit variational inference \citep{yin2018semi}
to flexibly construct expressive encoders.
Moreover, the objective \eqref{eq:pvae} can also adopt
importance weighted autoencoders \citep{burda2015importance,mattei2019miwae}
to optimize a tighter variational bound.

\subsection{Partial Bidirectional GAN}
\label{sec:pbigan}

Unlike P-VAE, which requires specifying an explicit density,
we can instead learn the distribution $p_\F(f)$
parameterized by \eqref{eq:gen complete}
implicitly based on generative adversarial networks (GANs)
\citep{goodfellow2014generative}.
Inspired by the Bidirectional GAN (BiGAN)
\citep{donahue2016bigan,dumoulin2017adversarially},
we propose a model that improves on MisGAN \citep{li2018learning}
for modeling incomplete data.
We call the proposed approach the \text{Partial Bidirectional GAN} (P-BiGAN).

The overall structure of P-BiGAN is shown in Figure~\ref{fig:pbigan},
which consists of a separate encoding and decoding part.
Given an incomplete dataset $\D=\{(\xx_i,\tt_i)\}_{i=1}^{n}$,
P-BiGAN aims to match
the joint distribution of the incomplete data $(\xx,\tt)$
sampled from $\D$
and the corresponding code $\zz$ drawn from $p_\phi(\zz|\xx,\tt)$ to
the joint distribution of generated masked outputs
$(g_\theta(\zz',\tt'), \tt')$ where $\zz'$ is a random latent code
drawn from the prior $p_z(\zz')$ and $\tt'$ is a set of random indices
separately sampled from $\D$.\footnote{
  Here $\tt'$ is essentially drawn from $p_\I(\tt)$,
  the marginal of $p_\D(\xx,\tt)\equiv p_\D(\xx)p_\I(\tt)$,
due to the independence assumption.}
Note that we use $g_\theta(\zz,\tt)$ as shorthand notation
for $[g_\theta(\zz,t_i)]_{i=1}^{|\tt|}$.

Specifically, P-BiGAN tries to solve
the following minimax optimization problem:
\begin{equation}
  \min_{\theta,\phi}\max_D L(D,\theta,\phi)
  \label{eq:bigan objective}
  \vspace{-.8em}
\end{equation}
where 
\begin{align*}
  L(D,\theta,\phi)&=
  \E_{(\xx,\tt)\sim p_\D}\E_{\zz\sim p_\phi(\zz|\xx,\tt)}\left[
      \log D(\xx,\tt,\zz)
  \right] \\
  &
  \!\!\!\!\!\!\!
  +\E_{\zz\sim p_z(\zz)}\E_{(\xx,\tt)\sim p_\D}\left[
    \log (1 - D(g_\theta(\zz,\tt),\tt,\zz))
  \right].
\end{align*}
P-BiGAN is compatible with many GAN variations.
Other form of $L(D,\theta,\phi)$
such as the loss used by BigBiGAN \citep{donahue2019large}
can also be applied.
The encoder of P-BiGAN can be constructed more flexibly than P-VAE
as we don't need to evaluate the density of the drawn samples.
For example, we can construct a distribution using the
generative process shown below
where the encoded samples are first drawn from a parameterized Gaussian
followed by a transformation $g_\phi$:
\[
  \vv=m(\xx,\tt),\
  \uu\sim\N(\mu_\phi(\vv),\Sigma_\phi(\vv)),\
  \zz=g_\phi(\uu).
\]


The discriminator of P-BiGAN takes as input
an incomplete data sample $(\xx,\tt)$ and its corresponding
code $\zz$.
Following MisGAN \citep{li2018learning},
the discriminator is constructed in the form of $D(m(\xx,\tt),\zz)$,
which can also be parameterized by neural networks.
Proposition~\ref{prop:marginals} below justifies the use of $m(\xx,\tt)$
when the data lies in a finite space,
under the independence assumption described in Section~\ref{sec:framework}.
\begin{proposition}
  \label{prop:marginals}
  (Adapted from \citet[Theorem~2]{li2018learning})
  When the data space and index set are both finite,
  given a distribution
  $p_\I(\tt)$,
  two distributions $p_\theta(f)$ and $p_{\theta'}(f)$
  induce the same distribution
  of $m(\xx,\tt)$
  if and only if they have the same marginals
  $p_\theta(\xx|\tt) = p_{\theta'}(\xx|\tt)$
  for all $\tt$ with $p_\I(\tt) > 0$.
\end{proposition}

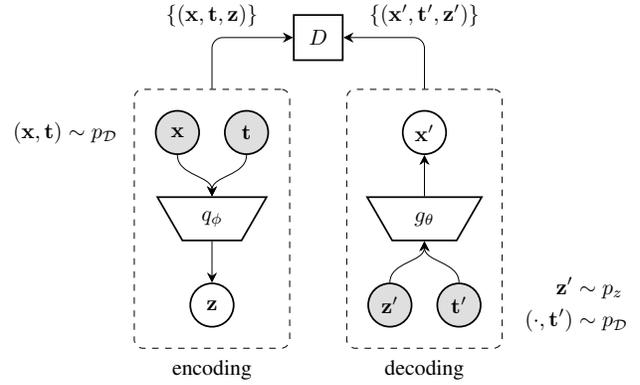
\begin{figure}
  \centering
  \scalebox{.82}{  \begin{tikzpicture}[scale=1.4, >=latex']
    \tikzstyle{every node} = [circle, minimum size=20pt, inner sep=0pt]
    \newcommand\aecolor{white}
    \newcommand\obscolor{gray!25}
    \newcommand\arrow{-Stealth[length=1.6mm,width=1.6mm]}
    \newcommand\labely{-1.75}
    \begin{scope}
      \newcommand\xdis{.8}
      \node[draw, thick, fill=\obscolor] (x) at (0, 1) {$\xx$};
      \node[draw, thick, fill=\obscolor] (t) at (\xdis, 1) {$\tt$};
      \node[draw, thick] (z) at (\xdis/2, -1) {$\zz$};
      \node [trapezium, trapezium angle=-60, minimum width=10, draw, thick]
      (q) at (\xdis/2, 0) {$q_\phi$};
      \draw[\arrow] (x) to[out=270, in=90] (q);
      \draw[\arrow] (q) -- (z);
      \draw[\arrow] (t) to[out=270, in=90] (q);
      \node[draw, fit={(-0.5, -1.5) (\xdis+.5, 1.5)}, rectangle,
      rounded corners, dashed, inner sep=0pt] (box1) {};
      \node[rectangle] at (\xdis/2, \labely) {encoding};
      \node[rectangle] at (-1.3, 1) {$(\xx,\tt)\sim p_\D$};
    \end{scope}
    \begin{scope}[xshift=7em]
      \newcommand\xdis{.8}
      \node[draw, thick] (x') at (\xdis/2, 1) {$\xx'$};
      \node[draw, thick, fill=\obscolor] (t') at (\xdis, -1) {$\tt'$};
      \node[draw, thick, fill=\obscolor] (z') at (0, -1) {$\zz'$};
      \node [trapezium, trapezium angle=-60, minimum width=10, draw, thick]
      (g') at (\xdis/2, 0) {$g_\theta$};
      \draw[\arrow] (g') -- (x');
      \draw[\arrow] (t') to[out=90, in=270] (g');
      \draw[\arrow] (z') to[out=90, in=270] (g');
      \node[draw, fit={(-0.5, -1.5) (\xdis+.5, 1.5)}, rectangle,
      rounded corners, dashed, inner sep=0pt] (box2) {};
      \node[rectangle] at (\xdis/2, \labely) {decoding};
      \node[rectangle] at (2.14, -1) {
        $\begin{aligned}
          \zz' &\sim p_z \\
          (\cdot,\tt') &\sim p_\D
        \end{aligned}$
      };
    \end{scope}
    \node [draw, thick, rectangle, inner sep=7pt]
    (d) at ($(box1)!.5!(box2) + (0,2.1)$) {$D$};
    \path [draw,\arrow,rounded corners=5pt] (box1.north) |-
    node[above, rectangle] {$\{(\xx,\tt,\zz)\}$} (d.west);
    \path [draw,\arrow,rounded corners=5pt] (box2.north) |-
    node[above, rectangle] {$\{(\xx',\tt',\zz')\}$} (d.east);

  \end{tikzpicture}}
  \vspace{-2em}
  \caption{The structure of P-BiGAN.}
  \label{fig:pbigan}
  \vspace{-1em}
\end{figure}

Moreover,
following \citet{donahue2016bigan},
the global optimum of \eqref{eq:bigan objective}
is achieved if and only if
the induced joint distribution over $\xx$, $\tt$ and $\zz$
are identical for the encoder $q_\phi(\zz|\xx,\tt)$ and decoder $g_\theta$.
We can show the following invertibility relationship
between the encoder and the decoder when optimality is attained
(see Appendix~\ref{sec:proof} for the proof).
\begin{proposition}
  \label{prop:inverse}
When the optimally learned encoder and decoder achieve the same joint
distribution over $(\xx,\tt)$ and $\zz$
by optimizing \eqref{eq:bigan objective},
for any $(\xx,\tt)$ with non-zero probability,
if $\zz\sim q_\phi(\zz|\xx,\tt)$ we have
$g_\theta(\zz,\tt) = \xx$ almost surely.

\end{proposition}

In practice, it is hard to achieve optimality with GAN training,
and therefore we usually don't have a very good match
between $g_\theta(\zz,\tt)$ and $\xx$ as described in
Proposition~\ref{prop:inverse}.
For applications that rely on the encoded representation $\zz$
such as those that we will present later in Section~\ref{sec:apps},
we found that further adding an autoencoding loss in addition
to the original P-BiGAN loss $L(D,\theta,\phi)$ to enforce this consistency
improves the results (see Appendix~\ref{sec:ae reg}).
Specifically, when training the model, we instead use the
following objective with some $\lambda \ge 0$ that controls the strength of
the autoencoding term:
\begin{align}
  L(D,\theta,\phi) + \lambda
  \E_{\zz\sim q_\phi(\zz|\xx,\tt)}\Bigg[
  \sum_{i=1}^{|\tt|} \ell\left(x_i, g_\theta(\zz,t_i)\right)\Bigg]
  \label{eq:p-bigan-ae}
\end{align}
where $\ell(x,x')$ is a loss function that measures the discrepancy
between $x$ and $x'$ such as $L_2$ loss for real-valued data,
which is analogous to the log likelihood term $\log p_\theta(x_i|\zz,t_i)$
in P-VAE.


Finally, we point out that there are two main differences between
P-BiGAN and MisGAN.
First, P-BiGAN utilizes the independence assumption
to sample $\tt'$ directly from the training data instead
of learning the distribution $p_\I$ as in MisGAN.
This not only makes the training faster, but improves the quality of
the resulting data generator when the distribution $p_\I$
is difficult to learn.
Second, the imputer in MisGAN can only be applied to data with
finite index set.
Since P-BiGAN is an encoder-decoder framework,
this not only greatly simplifies the model complexity
but can be generalized to the case of continuous index sets
as we discuss in the next section.

\section{Irregularly-Sampled Time Series:
The Continuous Index Set Case}
\label{sec:cont time}

For continuous time series defined over some time interval $[0, T]$,
the index set $\I=[0,T]$ is no longer finite.
In this section, we propose a computationally efficient encoder-decoder architecture for
modeling irregularly-sampled time series data.

\subsection{Decoder: Kernel Smoother}
\label{sec:kernel smoother}

To model the distribution of continuous functions over
the time interval $[0, T]$,
we first use a standard convolutional neural network (CNN) decoder
to generate a length-$L$ output
$v_1,\dots,v_L$
as the reference values on a set of evenly-spaced locations $u_1,\dots,u_L$
over $[0, T]$, and then construct the function as the smooth interpolation
of those references.
Here we use a kernel smoother to interpolate at arbitrary times.
Specifically, we model irregularly-sampled time series
as samples from a distribution over functions defined by
the following generative process:
\begin{equation}
\begin{split}
  \zz &\sim p_z(\zz), \\
  \vv &= \text{\textsc{CNN}}_\theta(\zz), \\
  f(t) &= \frac{\sum_{i=1}^L K(u_i,t)v_i}{\sum_{i=1}^L K(u_i,t)}
\end{split}
\label{eq:cont decoder}
\end{equation}
where $K$ is a smoothing kernel.
We use the Epanechnikov kernel,
$K(u,t) = \max(3/4 (1 - (|u-t| / \beta)^2), 0)$,
which has finite support so that each location is only
influenced by a small number of its neighbors.
Moreover, we can compute the kernel values among those neighbors only once
in the beginning as those stay constant during training.

This kernel smoother layer can also be applied to
multivariate time series by interpolating each channel independently
using the kernel smoother on a CNN with multi-channel output.\footnote{
For multivariate time series
with $C$ channels defined over the time interval $[0, T]$, the index set
$\I=\{1,\dots,C\} \times [0, T]$.}

\subsection{Encoder: Continuous Convolutional Layer}
\label{sec:cont conv}

Inspired by CNNs, we adapt
the convolutional layer in CNNs to accommodate irregularly-sampled time series.
To mimic the locally-focused receptive field of
standard convolutional layers,
we generalize the discrete filter (or kernel) to a continuous function $w(t)$
defined over a fixed small interval, say $[0, h]$
with a tunable kernel width $h$.
That is, $w(t)=0$ when $t \notin [0, h]$.

Similar to the convolutional layers in CNNs,
we perform
cross-correlation between the continuous filter $w(t)$ and
the masked function $f(t) = \sum_{i=1}^{|\tt|} x_i \delta(t-t_i)$
induced by the observations in time series
as follows, where $\delta(\cdot)$ is the Dirac delta function:\footnote{
  The function $f(t) = \sum_{i=1}^{|\tt|} x_i \delta(t-t_i)$
  defined over $[0,T]$
  is the analogy of the masked function $m(\xx,\tt)$
  in Definition~\ref{def:masking} for the case of continuous index set.}
\begin{align*}
  (w\star f)(r)
  &= \int w(t-r)\bigg(\sum_{i=1}^{|\tt|} x_i \delta(t-t_i)\bigg)dt \\
  &= \sum_{i:\,t_i-r \in[0,h]} w(t_i-r) x_i.
\end{align*}
We apply this operation on $L$ (need not be the same $L$ for the decoder)
evenly-spaced locations
$r_1,\dots,r_L$ spanning the time interval $[0,T]$ to transform
non-uniform inputs to a length-$L$ uniform representation
$[(w\star f)(r_i)]_{i=1}^L$.

We construct the continuous filter $w(t)$
as a piecewise linear function parameterized by
a small number of evenly-spaced knots over $[0, h]$.
This is equivalent to a degree-1 B-spline \citep{piegl2012nurbs}
and backpropagation through such functions can be computed efficiently
\citep{Fey/etal/2018}.
We found that degree-1 B-splines already perform well comparing
with more expensive higher-order B-spline interpolation.

In preliminary experiments,
we compared this architecture with several alternatives.
First, we use a multi-layer perceptron (MLP) to approximate
an arbitrary function as in neural processes
\citep{garnelo2018conditional,garnelo2018neural}.
However, an MLP is not as parameter efficient as
a piecewise linear function whose only parameters are the values of the knots.
We found that we need many  more parameters for an MLP to achieve similar
performance to piecewise linear functions
and the optimization is generally more difficult.
We also compare with a kernel smoother similar to the decoder described in
Section~\ref{sec:kernel smoother} to provide another parameter efficient choice.
Although a kernel smoother gives roughly the same performance,
it is about 20\% slower than the piecewise linear function
due to the expensive normalization.
Note that
although not as efficient as the convolutional structure,
we can also construct the encoder with the attention mechanism
such as in \citet{kim2019attentive} and \citet{lee2019set}.

We can extend this operator to the
case when there are $C_{\text{in}}$ input channels
and $C_{\text{out}}$ output channels.
Given a multi-channel incomplete example
$(\xx,\tt)=\{(\xx_c,\tt_c)\}_{c=1}^{C_{\text{in}}}$,
we define the continuous convolutional layer as
\[
  \text{\textsc{Conv}}_k(r,\xx,\tt) = b_k + \sum_{c=1}^{C_{\text{in}}}\;
  \sum_{i:\,t_{c,i} - r \in[0, h]} w_{c,k}(t_{c,i}-r) x_{c,i}
\]
where a bias term $b_k$ is included similar to standard convolutional layers.
For each time series,
the continuous convolutional layer produces a 2D output
$\VV\in\R^{C_\text{out}\times L}$
where $V_{kj}=\text{\textsc{Conv}}_k(r_j,\xx,\tt)$,
which can then be fed into a regular CNN encoder.
Note that \textsc{Conv} is a permutation invariant function
like the encoders mentioned in Section~\ref{sec:ae}.

Similar to the kernel smoother,
we can also precompute the distance to the neighboring reference points
once in the beginning for the continuous convolutional layer.
Note that the same architecture can also be used for the discriminator
in P-BiGAN.

\section{Applications}
\label{sec:apps}
\begin{figure*}[th]
  \centering
  \def\imgw{.42\textwidth}
  \def\vgap{2em}
  \def\hgap{2.5em}
  \begin{tikzpicture}
    \node[inner sep=0pt,
      label={below:{\footnotesize{
      MNIST: independent dropout with 90\% missing}}}
    ] (r1) { \includegraphics[width=\imgw]{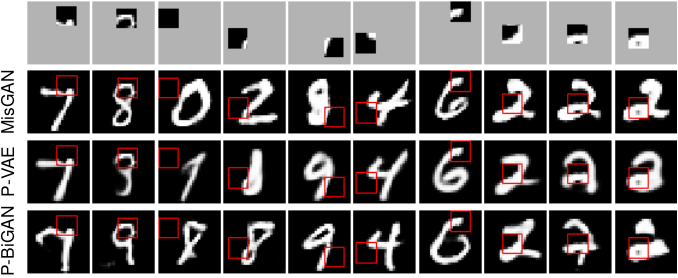}};
    \node[inner sep=0pt,right=\hgap of r1,
      label={below:{\footnotesize{
      MNIST: independent dropout with 90\% missing}}}
    ] (r2) { \includegraphics[width=\imgw]{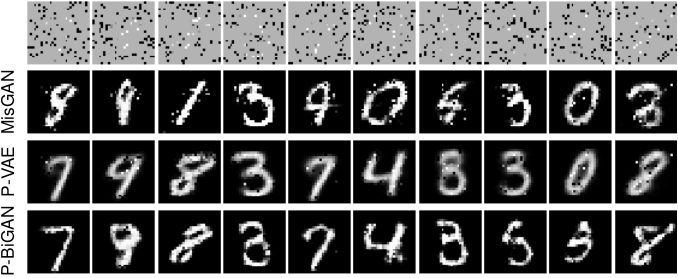}};
    \node[inner sep=0pt,below=\vgap of r1,
      label={below:{\footnotesize{
      CelebA: square observation with 90\% missing}}}
    ] {\includegraphics[width=\imgw]{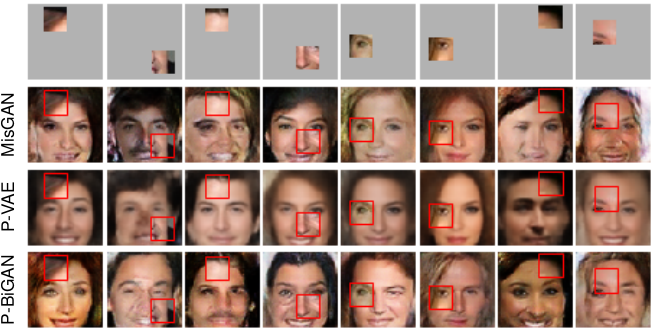}};
    \node[inner sep=0pt,below=\vgap of r2,
      label={below:{\footnotesize{
      CelebA: independent dropout with 90\% missing}}}
    ] {\includegraphics[width=\imgw]{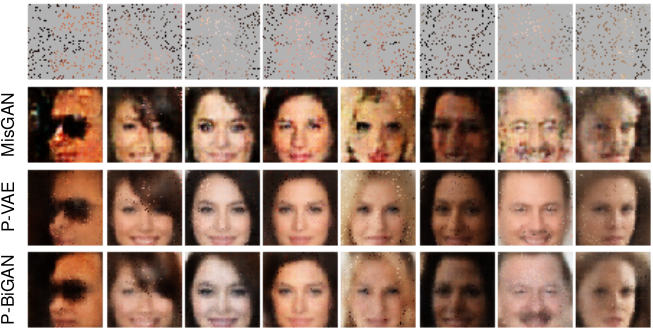}};
  \end{tikzpicture}
  \caption{Imputation results on MNIST and CelebA under 90\% missingness.
  The images in first row of each block are the incomplete images
  where gray pixels indicate missing data.
  For square observation cases on the left,
  the pixels inside of each red box are observed.}
  \label{fig:imputation}
\end{figure*}

In this section, we briefly describe two applications of our encoder-decoder
model framework: missing data imputation and supervised learning.

\subsection{Missing Data Imputation}
\label{sec:imputation}
Given an incomplete example $(\xx,\tt)$,
the goal of missing data imputation is to infer the values of the unobserved
features $\xx'$ that correspond
to indices $\tt'\subseteq\I\setminus\tt$ according to
$p(\xx'|\tt',\xx,\tt)$.
Once the model is trained,
imputations can be drawn according to the distribution
\[
  p(\xx'|\tt',\xx,\tt)
  = \E_{\zz\sim q_\phi(\zz|\xx,\tt)} \left[p_\theta(\xx'|\zz,\tt')\right].
\]
Since $p_\theta(\xx'|\zz,\tt')$ is defined implicitly by
\eqref{eq:gen complete},
sampling from $p(\xx'|\tt',\xx,\tt)$ can be done with the following steps:
\begin{align*}
  \zz \sim q_\phi(\zz|\xx,\tt),\quad
  f = g_\theta(\zz),\quad
  \xx' = [f(t_i')]_{i=1}^{|\tt'|}.
\end{align*}
%


\subsection{Supervised Learning}

We can perform supervised learning when each incomplete data case
has a corresponding prediction target.
We focus on the classification
case where the prediction target is a class label $y$.
We assume $y$ depends only on the
latent representation $\zz$ in the generative process \eqref{eq:gen complete}.


For P-VAE,
we augment the training objective to include the
classification term $p(y|\zz)$ as  follows:
\begin{align}
  &\E_{\zz\sim q_\phi(\zz|\xx,\tt)}\left[
    \log
    \frac{p_z(\zz) p(y|\zz) p_\theta(\xx|\zz,\tt)}
    {q_\phi(\zz|\xx,\tt)}
  \right]
  \label{eq:classification}
  \\
  &=
  \underbrace{\vphantom{\Bigg(}
    \E_{q_\phi(\zz|\xx,\tt)}\!\left[
      \log\frac{p_z(\zz) p_\theta(\xx|\zz,\tt)}{q_\phi(\zz|\xx,\tt)}
  \right]
  }_\text{regularization}
  +
  \underbrace{\vphantom{\Bigg(}
    \E_{q_\phi(\zz|\xx,\tt)}\!\left[\log p(y|\zz)\right]
  }_\text{classification}
  .
  \notag
\end{align}

Note that we use the encoder that depends only on the incomplete
data instead of the most general form $q_\phi(\zz|\xx,\tt,y)$,
which includes the class label as well.
This allows us to decompose \eqref{eq:classification} into
two separate terms: a regularization term as in P-VAE and
a classification term $\E_{q_\phi(\zz|\xx,\tt)}\!\left[\log p(y|\zz)\right]$.
Therefore, we can either train the classifier $p(y|\zz)$
along with the pre-trained encoder $q_\phi(\zz|\xx,\tt)$
or train the whole model jointly from scratch.
Moreover, this decomposition allows us to do semi-supervised learning
easily: we only include the classification term when the label
is available.

Similarly, for P-BiGAN,
we can train a classifier separately with the pre-trained encoder or
add a classification loss
$-\E_{q_\phi(\zz|\xx,\tt)}\!\left[\log p(y|\zz)\right]$
into \eqref{eq:p-bigan-ae} to jointly train the classifier with P-BiGAN.

Once the model is trained,
prediction can be performed efficiently
with the expectation approximated using a small number of samples
($S=1$ suffices in practice):
\begin{align*}
  y^*
  &=\argmax_y \E_{\zz\sim q_\phi(\zz|\xx,\tt)}\left[\log p(y|\zz)\right] \\
  &\approx\argmax_y \frac{1}{S}\sum_{s=1}^S \log p(y|\zz_s),
  \ \text{where $\zz_s \sim q_\phi(\zz|\xx,\tt)$}.
\end{align*}

\section{Experiments}
\label{sec:experiments}

In this section, we first evaluate the models on
the finite index set case described in
Section~\ref{sec:framework}.
We assess our framework using image modeling and completion experiments
with controlled missingness
on standard image benchmarks.
Next, we evaluate the performance of our framework equipped with
the continuous-time encoder/decoder
using the multivariate irregularly-sampled time series classification task
on a medical benchmark.
Additional results on time series imputation and visualization
of the learned temporal process on synthetic data
are provided in Appendix~\ref{sec:synthetic}.

\begin{figure*}[t!]
  \centering
  \includegraphics[width=\textwidth]{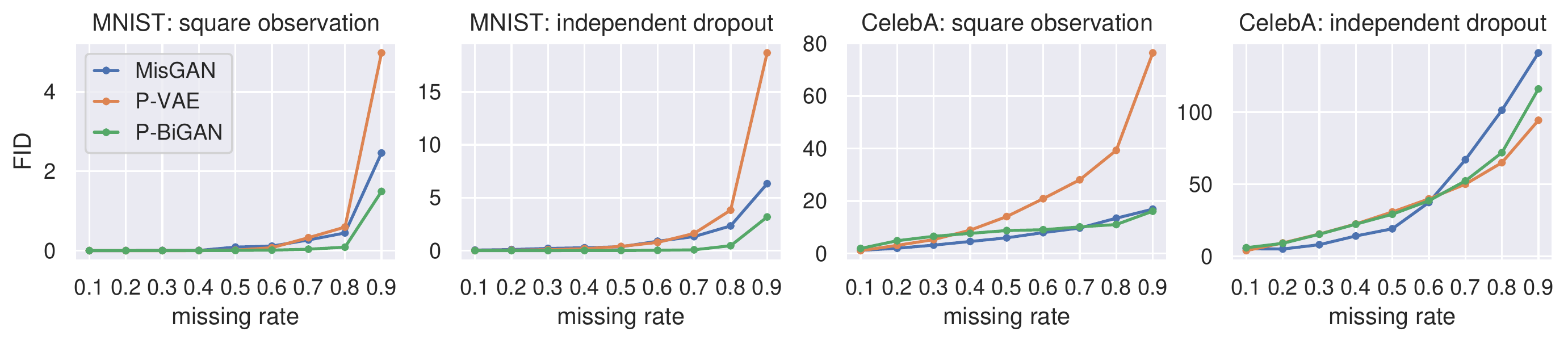}
  \vspace{-1.5em}
  \caption{Comparison of FIDs (the lower the better) on MNIST and CelebA with
  different missing patterns and missing rates.}
  \label{fig:fid}
\end{figure*}

\subsection{Image Modeling and Completion}
\label{sec:impute fid}

MisGAN was previously shown to outperform a range of methods
on the problem of learning the image distribution from incomplete data.
We follow the experimental setup of MisGAN
to quantitatively evaluate the performance of P-VAE and P-BiGAN on
the imputation task using two image benchmarks,
MNIST \citep{lecun2010mnist}
and CelebA \citep{liu2015faceattributes}.
We train the models using incomplete images
under two missing patterns:
i)~square observation where all pixels are missing
except for a square occurring at a random location on the image,
and ii)~independent dropout where each pixel is independently missing with
a given probability.
For both missing patterns,
we vary the missing rate from 10\% to 90\%.

To evaluate the quality of a model,
we impute all the incomplete images with the observed pixels
kept intact and use
the Fr\'echet Inception Distance (FID) \citep{heusel2017gans}
between the completed images
and the original fully-observed dataset as the evaluation metric.\footnote{
  Unlike FID that evaluates distributional discrepancy,
  metrics like RMSE that measure the
  discrepancy of imputation against the ground truth are not suitable here
  when the true posterior is highly
  multimodal, especially in the cases with high missingness.
  See \citet[Appendix~C]{li2018learning} for more details.
}

For P-VAE and P-BiGAN,
we use the same convolutional decoder architecture used in MisGAN.
For P-VAE,
we use an encoder $q_\phi(\zz|\xx,\tt)$ constructed by
\begin{equation}
  \zz_0 \sim \N(\mu_\phi(m(\xx,\tt)),\Sigma_\phi(m(\xx,\tt))),\
  \zz = \text{IAF}_\phi(\zz_0)
  \label{eq:iaf}
\end{equation}
using convolutional $\mu_\phi$ and $\Sigma_\phi$.
Here we use two-layers of inverse autoregressive flow (IAF)
\citep{kingma2016improved}.
In addition,
we use importance weighted autoencoders (IWAEs) with five importance weights.

For P-BiGAN, we use the same architecture as P-VAE including the IAF
component except we do not compute its density.
For the discriminator, we concatenate the embedding of $(\xx,\tt)$
computed using the same convolutional architecture
as the encoder and the embedding of $\zz$ using a two-layer MLP.
The concatenated embedding is then fed into another two-layer MLP
to produce the score.

Figure~\ref{fig:fid} compares the FIDs of MisGAN, P-VAE and P-BiGAN
under different missing patterns and missing rates.
For MNIST, it shows that P-BiGAN performs slightly better than MisGAN
due to the more expressive encoder architecture of P-BiGAN.
P-VAE has the worst FID scores especially for high missing rates,
which is reflected by the blurriness of the imputation results
shown in Figure~\ref{fig:imputation}.

For square observations on CelebA, P-BiGAN and MisGAN perform about the same,
while P-VAE has significantly worse FIDs also due to the blurriness.
However, for the independent dropout case,
P-VAE performs the best when the missing rate is high.
It seems that GAN-based models are better at capturing spatial correlations
when learning with convolutional networks, but
when neighboring pixels rarely co-occur, they are not able to learn effectively.
Because of the autoencoding regularization used in P-BiGAN,
it shares the benefit of autoencoding
when it comes to independent dropout and thus also outperforms
MisGAN when the missing rate is high.
However, for low missingness, MisGAN outperforms both P-VAE and P-BiGAN
due to its U-Net imputer that allows
the model to produce better imputation results when the images are
almost fully observed.







\begin{table}[t]
  \vspace{-.08in}
  \caption{
    The average per-epoch running time in minutes and the number of
    parameters of each model.
  }
  \vskip 0.12in
  \centering
  \begin{tabular}{clrr}
    \toprule
    dataset & method & time & params\\
    \midrule
    \multirow{3}{*}{MNIST}
    & MisGAN & 1.72 & 8.67M \\
    & P-VAE & 0.84 & 4.70M \\
    & P-BiGAN & 1.38 & 6.01M \\
    \midrule
    \multirow{3}{*}{CelebA}
    & MisGAN & 39.47 & 40.35M \\
    & P-VAE & 11.93 & 11.32M \\ 
    & P-BiGAN & 14.78 & 16.71M \\ 
    \bottomrule
  \end{tabular}
  \label{tab:images}
\end{table}

Table~\ref{tab:images} shows the per-epoch running time
and the number of parameters of each model,
where the running time is roughly proportional to the number of parameters.
For MNIST,
it shows that P-BiGAN and P-VAE have proportionally less parameters
then MisGAN, even if they both use a large encoder that
roughly doubles the parameters of the decoder.

For CelebA,
MisGAN uses a separate U-Net imputer trained with another discriminator,
while P-BiGAN only utilizes an additional encoder to impute along with
the decoder.
Moreover,
P-BiGAN does not model the missingness that requires
an extra pair of generator and discriminator for the masks as in MisGAN.
Therefore, the great reduction in model parameters makes
P-BiGAN about 2.7 times faster than MisGAN.
On the other hand,
P-VAE enjoys the simplest training procedure and the lowest model complexity
without the need for learning separate discriminators.
As a result, it is the fastest among the three models.



\subsection{Classification of Irregularly-Sampled Time Series}
\label{sec:mimic3}

In this section, we evaluate our framework on a healthcare multivariate
time series dataset,
MIMIC-III \citep{mimiciii},
using the mortality prediction task.

MIMIC-III consists of about 53,000 data cases. We use 12
irregularly-sampled temporal variables that are recorded within 48 hours.
If we discretize observations into 1-minute intervals,
the overall missing rate is about 92\% on average.
We rescale the timestamps within the 48-hour window to $[0, 1]$.
Our task is to predict the in-hospital mortality as a binary classification
problem.
We use the area under the ROC curve (AUC) as the evaluation metric.
We split the data into 64\% for training,
16\% for validation, and the remaining 20\% for testing.

We evaluate both P-VAE and P-BiGAN equipped with the continuous encoder and
decoder described in Section~\ref{sec:cont time},
which we denote Cont P-VAE and Cont P-BiGAN respectively.
For the decoder, we use 128 evenly-spaced references over
$[0,1]$ for the kernel smoother,
whose values are the output of a standard CNN decoder.
We use the Epanechnikov kernel with the kernel bandwidth set to
3/128.
For the continuous convolutional layer in the encoder,
we use 64 output channels with 98 evenly-spaced references.
The piecewise-linear convolutional kernel has width 2/98 with
7 knots. The output of the continuous convolutional layer is then
fed into a standard IAF encoder as in \eqref{eq:iaf}.

For Cont P-VAE we use 8 importance weights for the IWAE objective.
Both P-VAE and P-BiGAN are trained with a separate
two-layer fully-connected classifier jointly.
As an ablation study, we compare our models with a classifier,
denoted Cont Classifier,
that combines the same encoder and classifier used in Cont P-VAE and
Cont P-BiGAN, but without adding extra Gaussian noise in the encoder.
We compare our models with two recent methods designed for
irregularly-sampled time series:
GRU-D \citep{che2018recurrent} and Latent ODE \citep{rubanova2019latent}.
We also include a baseline model for
learning time series with missing data,
Multi-directional RNN (M-RNN) \citep{yoon2018estimating}.
Since M-RNN does not work well with massive missingness,
it is run on the modified data with
observations quantized into 30-minute intervals.

\begin{table}
  \vspace{-.08in}
  \caption{AUC (mean $\pm$ std), per-epoch time in hours and
  the number of parameters of each model on MIMIC-III.}
  \vskip 0.12in
  \centering
    \begin{tabular}{lccr}
      \toprule
      method & AUC (\%) & time & params \\
      \midrule
      M-RNN & 83.87 $\pm$ 0.80 & -- & 101.6K \\
      GRU-D & 83.88 $\pm$ 0.65 & 0.11 & 2.6K \\
      Latent ODE & {85.71 $\pm$ 0.38} & 2.62 & 154.7K \\
      \midrule
      Cont Classifier & 84.87 $\pm$ 0.18 & 0.03 & 30.5K \\
      Cont P-VAE & 85.52 $\pm$ 0.54 & 0.05 & 67.8K \\
      Cont P-BiGAN & {86.05 $\pm$ 0.36} & 0.22 & 73.2K \\
      \bottomrule
    \end{tabular}
    \label{tab:mimic3}
\end{table}

Table~\ref{tab:mimic3} shows predictive performance,
per-epoch training time and model sizes.
The training time of M-RNN is omitted because it runs on the
much smaller quantized dataset and the time is thus not
comparable to other methods.
The table shows that Cont P-BiGAN achieves the highest mean AUC
of all of the methods,
followed closely by Latent ODE and Cont P-VAE.
Although the difference between Cont P-BiGAN and Latent ODE
is not statistically significant,
Cont P-BiGAN is over 10 times faster per training epoch.
On the other hand, Cont P-VAE is over 50 times faster than Latent ODE.

These run time differences are due to the fact that
the convolutional architectures used
in the proposed approaches are highly parallelizable
compared to the recurrent structure used by the baseline models.
Moreover, our models directly parameterize temporal functions using
\eqref{eq:cont decoder};
on the contrary, Latent ODE instead models the dynamics using ODEs,
which requires expensive numerical integration.
Meanwhile, Cont P-VAE is faster than Cont P-BiGAN
because Cont P-BiGAN requires
running continuous convolutional layers
in both the encoder and discriminator,
which is the most expensive computation during training
that marshals time series of variable size.

We also note that Cont Classifier achieves better AUCs
than M-RNN and GRU-D without generative modeling.
This shows that the continuous convolutional layer provides
an effective intermediate representation
for irregularly-sampled time series.

\section{Conclusion}
\label{sec:discussion}

In this paper, we have presented the problem of modeling
irregularly-sampled time series from the perspective of missing data.
We introduced an encoder-decoder framework
for modeling general missing data problems
and introduced two model families leveraging this framework:
P-VAE and P-BiGAN.
We showed how to integrate this framework with
a continuous convolutional layer
to efficiently featurize irregularly-sampled time series
for interfacing with standard neural network architectures.
Our proposed models achieve comparable predictive performance
to the recently-proposed Latent ODE model,
while offering significantly faster training times.

\bibliography{ref}

\begin{thebibliography}{29}
\providecommand{\natexlab}[1]{#1}
\providecommand{\url}[1]{\texttt{#1}}
\expandafter\ifx\csname urlstyle\endcsname\relax
  \providecommand{\doi}[1]{doi: #1}\else
  \providecommand{\doi}{doi: \begingroup \urlstyle{rm}\Url}\fi

\bibitem[Burda et~al.(2016)Burda, Grosse, and
  Salakhutdinov]{burda2015importance}
Burda, Y., Grosse, R., and Salakhutdinov, R.
\newblock Importance weighted autoencoders.
\newblock In \emph{International Conference on Learning Representations
  (ICLR)}, 2016.

\bibitem[Che et~al.(2018)Che, Purushotham, Cho, Sontag, and
  Liu]{che2018recurrent}
Che, Z., Purushotham, S., Cho, K., Sontag, D., and Liu, Y.
\newblock Recurrent neural networks for multivariate time series with missing
  values.
\newblock \emph{Scientific reports}, 8\penalty0 (1):\penalty0 6085, 2018.

\bibitem[Chen et~al.(2018)Chen, Rubanova, Bettencourt, and
  Duvenaud]{chen2018neural}
Chen, T.~Q., Rubanova, Y., Bettencourt, J., and Duvenaud, D.~K.
\newblock Neural ordinary differential equations.
\newblock In \emph{Advances in Neural Information Processing Systems}, pp.\
  6571--6583, 2018.

\bibitem[Chung et~al.(2015)Chung, Kastner, Dinh, Goel, Courville, and
  Bengio]{chung2015recurrent}
Chung, J., Kastner, K., Dinh, L., Goel, K., Courville, A.~C., and Bengio, Y.
\newblock A recurrent latent variable model for sequential data.
\newblock In \emph{Advances in Neural Information Processing Systems}, pp.\
  2980--2988, 2015.

\bibitem[Donahue \& Simonyan(2019)Donahue and Simonyan]{donahue2019large}
Donahue, J. and Simonyan, K.
\newblock Large scale adversarial representation learning.
\newblock In \emph{Advances in Neural Information Processing Systems}, pp.\
  10542--10552, 2019.

\bibitem[Donahue et~al.(2017)Donahue, Kr\"ahenb\"uhl, and
  Darrell]{donahue2016bigan}
Donahue, J., Kr\"ahenb\"uhl, P., and Darrell, T.
\newblock Adversarial feature learning.
\newblock In \emph{International Conference on Learning Representations
  (ICLR)}, 2017.

\bibitem[Dumoulin et~al.(2017)Dumoulin, Belghazi, Poole, Lamb, Arjovsky,
  Mastropietro, and Courville]{dumoulin2017adversarially}
Dumoulin, V., Belghazi, M. I.~D., Poole, B., Lamb, A., Arjovsky, M.,
  Mastropietro, O., and Courville, A.
\newblock Adversarially learned inference.
\newblock In \emph{International Conference on Learning Representations
  (ICLR)}, 2017.

\bibitem[Fey et~al.(2018)Fey, Lenssen, Weichert, and M{\"u}ller]{Fey/etal/2018}
Fey, M., Lenssen, J.~E., Weichert, F., and M{\"u}ller, H.
\newblock {SplineCNN}: Fast geometric deep learning with continuous {B}-spline
  kernels.
\newblock In \emph{IEEE Conference on Computer Vision and Pattern Recognition
  (CVPR)}, 2018.

\bibitem[Garnelo et~al.(2018{\natexlab{a}})Garnelo, Rosenbaum, Maddison,
  Ramalho, Saxton, Shanahan, Teh, Rezende, and Eslami]{garnelo2018conditional}
Garnelo, M., Rosenbaum, D., Maddison, C., Ramalho, T., Saxton, D., Shanahan,
  M., Teh, Y.~W., Rezende, D., and Eslami, S.~A.
\newblock Conditional neural processes.
\newblock In \emph{International Conference on Machine Learning (ICML)}, pp.\
  1690--1699, 2018{\natexlab{a}}.

\bibitem[Garnelo et~al.(2018{\natexlab{b}})Garnelo, Schwarz, Rosenbaum, Viola,
  Rezende, Eslami, and Teh]{garnelo2018neural}
Garnelo, M., Schwarz, J., Rosenbaum, D., Viola, F., Rezende, D.~J., Eslami, S.,
  and Teh, Y.~W.
\newblock Neural processes.
\newblock \emph{arXiv preprint arXiv:1807.01622}, 2018{\natexlab{b}}.

\bibitem[Goodfellow et~al.(2014)Goodfellow, Pouget-Abadie, Mirza, Xu,
  Warde-Farley, Ozair, Courville, and Bengio]{goodfellow2014generative}
Goodfellow, I., Pouget-Abadie, J., Mirza, M., Xu, B., Warde-Farley, D., Ozair,
  S., Courville, A., and Bengio, Y.
\newblock Generative adversarial nets.
\newblock In \emph{Advances in Neural Information Processing Systems}, pp.\
  2672--2680, 2014.

\bibitem[Heusel et~al.(2017)Heusel, Ramsauer, Unterthiner, Nessler, and
  Hochreiter]{heusel2017gans}
Heusel, M., Ramsauer, H., Unterthiner, T., Nessler, B., and Hochreiter, S.
\newblock {GANs} trained by a two time-scale update rule converge to a local
  nash equilibrium.
\newblock In \emph{Advances in Neural Information Processing Systems}, pp.\
  6629--6640, 2017.

\bibitem[Johnson et~al.(2016)Johnson, Pollard, Shen, Li-wei, Feng, Ghassemi,
  Moody, Szolovits, Celi, and Mark]{mimiciii}
Johnson, A.~E., Pollard, T.~J., Shen, L., Li-wei, H.~L., Feng, M., Ghassemi,
  M., Moody, B., Szolovits, P., Celi, L.~A., and Mark, R.~G.
\newblock {MIMIC-III}, a freely accessible critical care database.
\newblock \emph{Scientific data}, 3:\penalty0 160035, 2016.

\bibitem[Kim et~al.(2019)Kim, Mnih, Schwarz, Garnelo, Eslami, Rosenbaum,
  Vinyals, and Teh]{kim2019attentive}
Kim, H., Mnih, A., Schwarz, J., Garnelo, M., Eslami, A., Rosenbaum, D.,
  Vinyals, O., and Teh, Y.~W.
\newblock Attentive neural processes.
\newblock In \emph{International Conference on Learning Representations
  (ICLR)}, 2019.

\bibitem[Kingma \& Welling(2014)Kingma and Welling]{kingma2013auto}
Kingma, D.~P. and Welling, M.
\newblock Auto-encoding variational bayes.
\newblock In \emph{Proceedings of the 2nd International Conference on Learning
  Representations (ICLR)}, 2014.

\bibitem[Kingma et~al.(2016)Kingma, Salimans, Jozefowicz, Chen, Sutskever, and
  Welling]{kingma2016improved}
Kingma, D.~P., Salimans, T., Jozefowicz, R., Chen, X., Sutskever, I., and
  Welling, M.
\newblock Improved variational inference with inverse autoregressive flow.
\newblock In \emph{Advances in Neural Information Processing Systems}, pp.\
  4743--4751, 2016.

\bibitem[LeCun et~al.(2010)LeCun, Cortes, and Burges]{lecun2010mnist}
LeCun, Y., Cortes, C., and Burges, C.
\newblock Mnist handwritten digit database.
\newblock \emph{ATT Labs [Online]. Available:
  \url{http://yann.lecun.com/exdb/mnist}}, 2, 2010.

\bibitem[Lee et~al.(2019)Lee, Lee, Kim, Kosiorek, Choi, and Teh]{lee2019set}
Lee, J., Lee, Y., Kim, J., Kosiorek, A., Choi, S., and Teh, Y.~W.
\newblock Set transformer: A framework for attention-based
  permutation-invariant neural networks.
\newblock In \emph{International Conference on Machine Learning (ICML)}, pp.\
  3744--3753, 2019.

\bibitem[Li et~al.(2019)Li, Jiang, and Marlin]{li2018learning}
Li, S. C.-X., Jiang, B., and Marlin, B.
\newblock {M}is{GAN}: Learning from incomplete data with generative adversarial
  networks.
\newblock In \emph{International Conference on Learning Representations
  (ICLR)}, 2019.

\bibitem[Little \& Rubin(2014)Little and Rubin]{little2014statistical}
Little, R.~J. and Rubin, D.~B.
\newblock \emph{Statistical analysis with missing data}, volume 333.
\newblock John Wiley \& Sons, 2014.

\bibitem[Liu et~al.(2015)Liu, Luo, Wang, and Tang]{liu2015faceattributes}
Liu, Z., Luo, P., Wang, X., and Tang, X.
\newblock Deep learning face attributes in the wild.
\newblock In \emph{Proceedings of International Conference on Computer Vision
  (ICCV)}, December 2015.

\bibitem[Ma et~al.(2018)Ma, Gong, Hern{\'a}ndez-Lobato, Koenigstein, Nowozin,
  and Zhang]{ma2018partial}
Ma, C., Gong, W., Hern{\'a}ndez-Lobato, J.~M., Koenigstein, N., Nowozin, S.,
  and Zhang, C.
\newblock Partial {VAE} for hybrid recommender system.
\newblock In \emph{NIPS Workshop on Bayesian Deep Learning. 2018}, 2018.

\bibitem[Ma et~al.(2019)Ma, Tschiatschek, Palla, Hernandez-Lobato, Nowozin, and
  Zhang]{ma2019eddi}
Ma, C., Tschiatschek, S., Palla, K., Hernandez-Lobato, J.~M., Nowozin, S., and
  Zhang, C.
\newblock {EDDI}: Efficient dynamic discovery of high-value information with
  partial {VAE}.
\newblock In \emph{International Conference on Machine Learning (ICML)}, pp.\
  4234--4243, 2019.

\bibitem[Mattei \& Frellsen(2019)Mattei and Frellsen]{mattei2019miwae}
Mattei, P.-A. and Frellsen, J.
\newblock {MIWAE}: Deep generative modelling and imputation of incomplete data
  sets.
\newblock In \emph{International Conference on Machine Learning (ICML)}, pp.\
  4413--4423, 2019.

\bibitem[Piegl \& Tiller(2012)Piegl and Tiller]{piegl2012nurbs}
Piegl, L. and Tiller, W.
\newblock \emph{The NURBS book}.
\newblock Springer Science \& Business Media, 2012.

\bibitem[Rubanova et~al.(2019)Rubanova, Chen, and Duvenaud]{rubanova2019latent}
Rubanova, Y., Chen, T.~Q., and Duvenaud, D.~K.
\newblock Latent ordinary differential equations for irregularly-sampled time
  series.
\newblock In \emph{Advances in Neural Information Processing Systems}, pp.\
  5321--5331, 2019.

\bibitem[Yin \& Zhou(2018)Yin and Zhou]{yin2018semi}
Yin, M. and Zhou, M.
\newblock Semi-implicit variational inference.
\newblock In \emph{International Conference on Machine Learning (ICML)}, pp.\
  5646--5655, 2018.

\bibitem[Yoon et~al.(2018)Yoon, Zame, and van~der Schaar]{yoon2018estimating}
Yoon, J., Zame, W.~R., and van~der Schaar, M.
\newblock Estimating missing data in temporal data streams using
  multi-directional recurrent neural networks.
\newblock \emph{IEEE Transactions on Biomedical Engineering}, 66\penalty0
  (5):\penalty0 1477--1490, 2018.

\bibitem[Zaheer et~al.(2017)Zaheer, Kottur, Ravanbakhsh, Poczos, Salakhutdinov,
  and Smola]{zaheer2017deep}
Zaheer, M., Kottur, S., Ravanbakhsh, S., Poczos, B., Salakhutdinov, R.~R., and
  Smola, A.~J.
\newblock Deep sets.
\newblock In \emph{Advances in Neural Information Processing Systems}, pp.\
  3391--3401, 2017.

\end{thebibliography}
\bibliographystyle{icml2020}

\newpage
\clearpage
\appendix
\label{appendix}
\section{Proof of Proposition~\ref{prop:inverse}}
\label{sec:proof}

\setcounter{proposition}{1}

\begin{proposition}
When the optimally learned encoder and decoder achieve the same joint
distribution over $(\xx,\tt)$ and $\zz$
by optimizing \eqref{eq:bigan objective},
for any $(\xx,\tt)$ with non-zero probability,
if $\zz\sim q_\phi(\zz|\xx,\tt)$ we have
$g_\theta(\zz,\tt) = \xx$ almost surely.

\end{proposition}

\begin{proof}
  The joint distribution induced by the encoder is
  \begin{equation*}
    p_\text{enc}(\xx,\tt,\zz) = p_\D(\xx,\tt)q_\phi(\zz|\xx,\tt).
  \end{equation*}
  The joint distribution induced by the decoder is
  \begin{equation*}
    p_\text{dec}(\xx,\tt,\zz) = p_\I(\tt)p_z(\zz)\delta(\xx-g_\theta(\zz,\tt)).
  \end{equation*}
  When the optimality is achieved so that
  $p_\text{enc}=p_\text{dec}$, for $p_\D(\xx,\tt) > 0$ we have
  \begin{equation*}
    q_\phi(\zz|\xx,\tt) =
    \frac{p_\I(\tt)p_z(\zz)}{p_\D(\xx,\tt)}\delta(\xx-g_\theta(\zz,\tt)).
  \end{equation*}
  Therefore, given $(\xx,\tt)$ such that $p_\D(\xx,\tt) > 0$,
  for $Z\sim q_\phi(\zz|\xx,\tt)$ we have
  \begin{align*}
    \operatorname{Pr}[\xx=g_\theta(Z,\tt)]
    &=\int\1\{\xx=g_\theta(\zz,\tt)\}q_\phi(\zz|\xx,\tt)d\zz \\
    &=\int q_\phi(\zz|\xx,\tt)d\zz \\
    &=1.
    \qedhere
  \end{align*}
\end{proof}

\section{On the Independence Assumption}
\label{sec:independence}

Throughout this paper, we assume
the complete temporal process $f$ and the observation indices $\tt$
are independent, which corresponds to the
missing completely at random (MCAR) case categorized by
\citet{little2014statistical}.
We point out that P-VAE is still unbiased if the data are
missing at random (MAR) according to \citet[Chapter~6]{little2014statistical}.

We note that the introduction of the independence assumption
is mainly for better modeling scalability and stability.
For the most general situation that corresponds to
the not missing at random (NMAR) case,
we will need to model the dependent index distribution
explicitly in both P-VAE and P-BiGAN.
One convenient choice is to model this distribution as $p_\I(\tt|\zz)$
that conditions on the common latent code $\zz$ shared with the data $\xx$,
which results in the following generative process:
\[
  \zz \sim p_z(\zz), \quad
  \tt \sim p_\I(\tt|\zz), \quad
  \xx = g_\theta(\zz,\tt).
\]
This encodes the dependency between $\tt$ and $\xx$ when $\zz$ is unobserved.
For P-VAE, we maximize the following expected variational lower bound on
$\log p(\xx,\tt)$ with additional model parameters for $p_\I(\tt|\zz)$:
\begin{equation*}
  \E_{(\xx,\tt)\sim p_\D}
  \E_{q_\phi(\zz|\xx,\tt)}\!\left[
    \log \frac{p_z(\zz) p_\I(\tt|\zz)\prod_{i=1}^{|\tt|} p_\theta(x_i|\zz,t_i)}
  {q_\phi(\zz|\xx,\tt)} \right].
\end{equation*}
For P-BiGAN, the minimax game becomes
\begin{align*}
  &\min_{\theta,\phi,\tau}\max_D
  \Big(
  \E_{(\xx,\tt)\sim p_\D}\E_{\zz\sim p_\phi(\zz|\xx,\tt)}\left[
      \log D(\xx,\tt,\zz)
  \right] \\
  &
  \qquad
  +\E_{\zz\sim p_z(\zz)}\E_{\tt\sim p_\I(\tt|\zz)}\left[
    \log (1 - D(g_\theta(\zz,\tt),\tt,\zz))
  \right]
  \Big)
\end{align*}
where $\tau$ denotes the parameters of $p_\I(\tt|\zz)$.
For P-BiGAN,
$p_\I(\tt|\zz)$ can be either stochastic or deterministic.

For time series,
we can use the variational RNN (VRNN)
\citep{chung2015recurrent} to model the temporal point process
$p_\I(\tt|\zz)$.
Specifically, at each step of VRNN that corresponds to an observation,
it outputs the duration until the next observation is made.
Our preliminary results show that incorporating VRNN $p_\I(\tt|\zz)$
makes learning the data distribution harder, especially for P-BiGAN
as the discriminator is sensitive to the discrepancy between
the learned temporal point process and the empirical samples of
observation times.
Specifically,
modeling the dependency of the temporal point process reduces bias while
significantly increasing variance such that the overall model ends up
performing worse.
The same phenomenon was also reported in the Latent ODE
work---\citet{rubanova2019latent}
jointly model a Poisson process using a Neural ODE,
which also leads to worse classification results.

Moreover, learning the temporal point process using variational RNN
is quite slow due to the sequential nature of RNNs.
It is challenging to model such distribution efficiently
given that the number of observations may be varied from case to case,
especially for P-BiGAN
that needs to discriminate samples of variable lengths.
Therefore,
studying how to effectively and efficiently learn
the temporal point process
and incorporate it in the missing data setting for time series
is of interest in the future.

%

\section{Autoencoding Regularization in P-BiGAN}
\label{sec:ae reg}

In Section~\ref{sec:pbigan} we discussed regularizing P-BiGAN
with an autoencoding loss using the augmented objective \eqref{eq:p-bigan-ae}.
Here we demonstrate the effect of introducing this autoencoding loss
in P-BiGAN by comparing the augmented model with
the non-regularized counterpart,
which is equivalent to the model with
the autoencoding coefficient $\lambda=0$.

Figure~\ref{fig:fid ae} compares P-BiGAN with the
default strictly-positive $\lambda$
and the one without autoencoding regularization using $\lambda=0$
on the MNIST and CelebA imputation experiments.
Similarly,
Table~\ref{tab:ae} compares P-BiGAN with the default $\lambda=1$
and the one without the autoencoding term
on the MIMIC-III experiment.
It shows that autoencoding regularization improves
the performance in almost all the cases.
Nonetheless, even without autoencoding regularization P-BiGAN
still gives reasonable imputation and classification results.
This provides empirical evidence to support the invertibility property
stated in Proposition~\ref{prop:inverse}.

\begin{figure}[hbt]
  \centering
  \hspace*{-1.5em}
  \includegraphics[width=.52\textwidth]{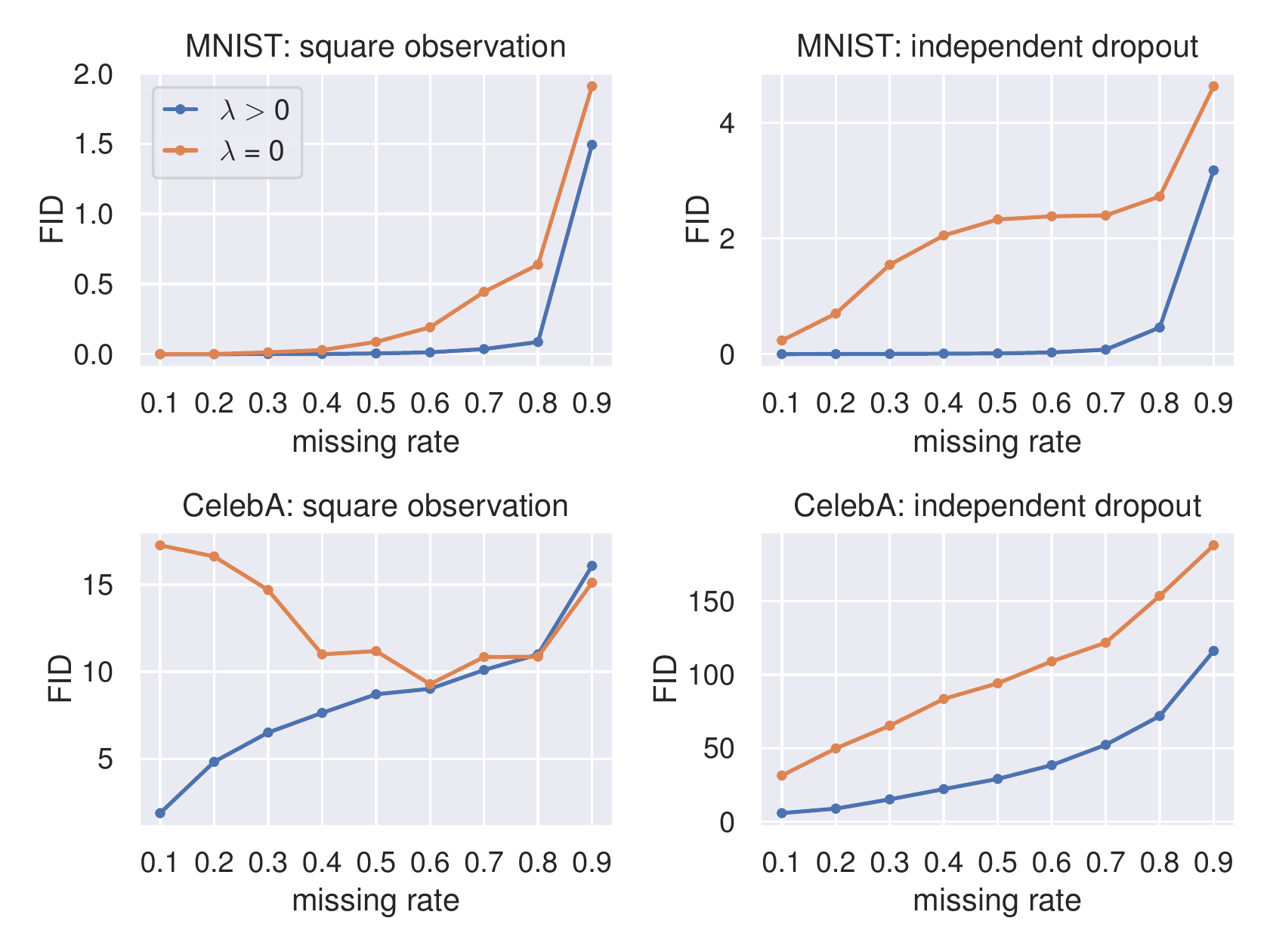}
  \vspace{-1.5em}
  \caption{Comparing the effect of autoencoding regularization on
  the imputation FIDs of P-BiGAN on MNIST and CelebA
  (no autoencoding regularization when $\lambda=0$).
  The high FIDs of the cases of low missing rates on
  CelebA with square observation are due to the inconsistency between
  the observed region and the imputed part.
  Figure~\ref{fig:fid gen ae} shows the FIDs of the generated images
  under the same settings,
  from which we can see that the decoder of P-BiGAN performs roughly
  the same regardless of the autoencoding regularization.
  }
  \label{fig:fid ae}
\end{figure}

\begin{figure}[hbt]
  \centering
  \hspace*{-1.5em}
  \includegraphics[width=.52\textwidth]{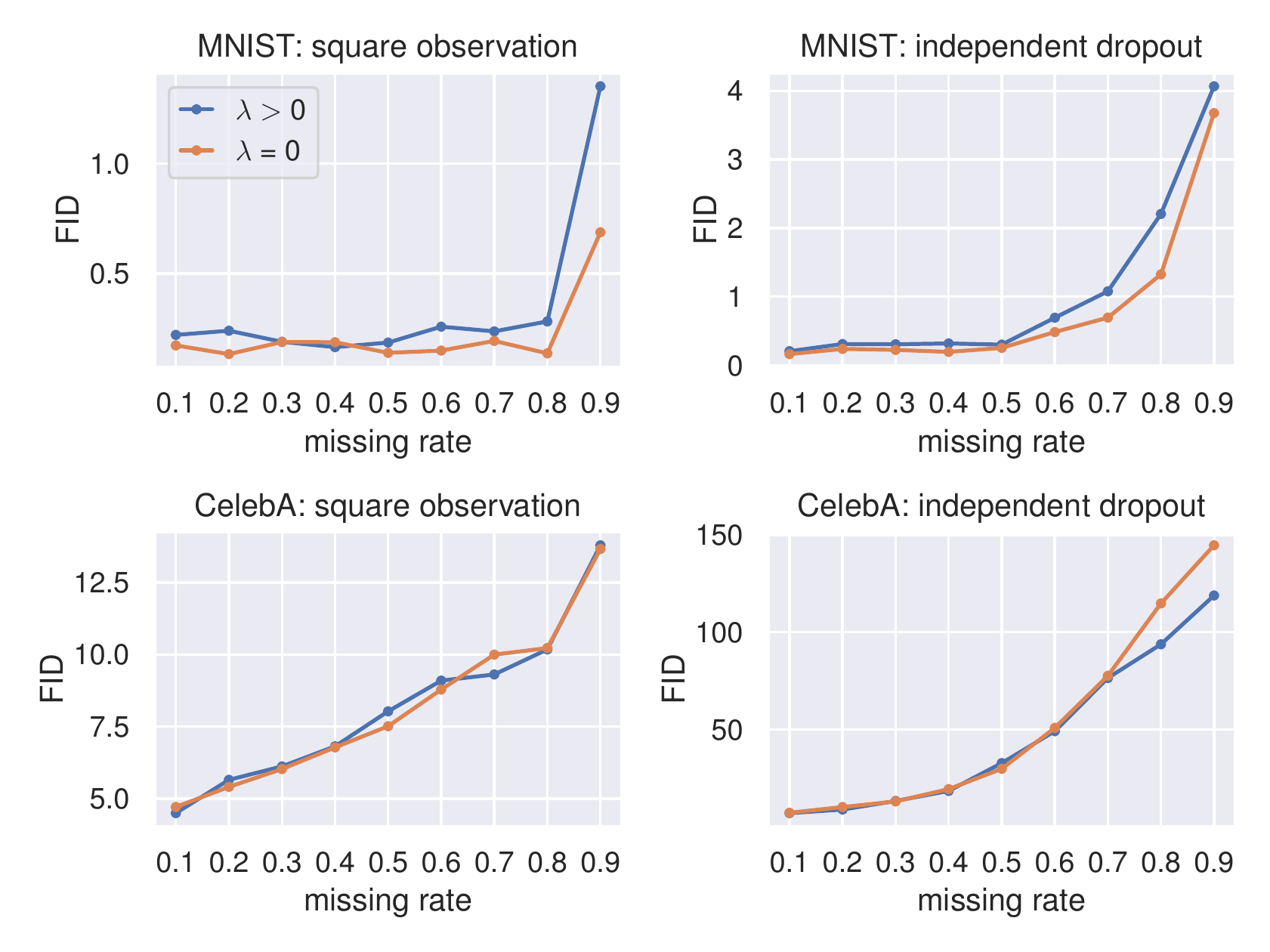}
  \vspace{-1.5em}
  \caption{Comparing the effect of autoencoding regularization on
  the generation FIDs of P-BiGAN on MNIST and CelebA
  (no autoencoding loss when $\lambda=0$).}
  \label{fig:fid gen ae}
\end{figure}

\begin{table}[t]
  \caption{
    Comparing P-BiGAN with autoencoding regularization ($\lambda=1$)
    and without it ($\lambda=0$) on MIMIC-III classification.
  }
  \vskip 0.12in
  \centering
  \begin{tabular}{cc}
    \toprule
    AE $\lambda$ & AUC (\%) \\
    \midrule
    $\lambda=0$ & 83.56 $\pm$ 0.49 \\
    $\lambda=1$ & 86.05 $\pm$ 0.36 \\
    \bottomrule
  \end{tabular}
  \label{tab:ae}
\end{table}

\section{Synthetic Multivariate Time Series}
\label{sec:synthetic}

\begin{figure*}
  \centering
  \def\vgap{.6em}
  \def\hgap{0em}
  \def\toydir{synthetic}
  \def\imgw{.24\textwidth}
  \begin{tikzpicture}
    \node[inner sep=0pt,
      label={[rotate=90,anchor=south]left:{\scriptsize\textsf{Data}}}
    ] (r11) {
      \includegraphics[width=\imgw]{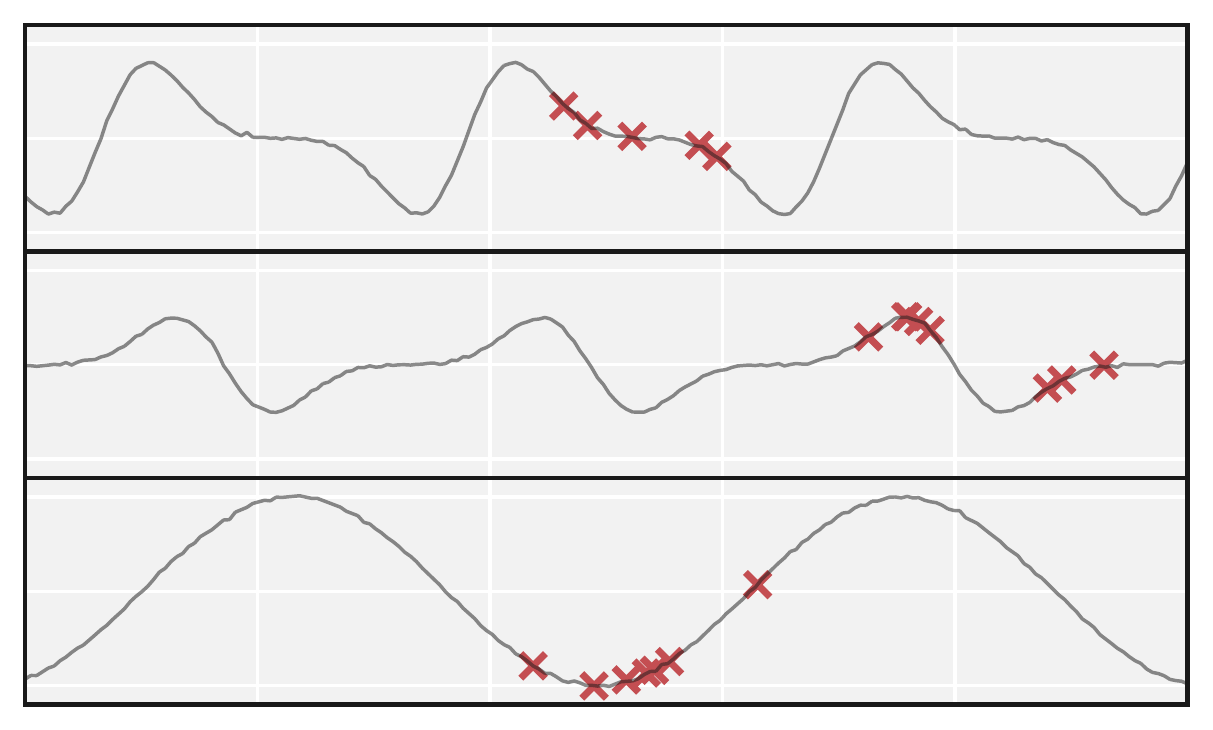}
    };
    \node[inner sep=0pt,right=\hgap of r11,
    ] (r12) {
      \includegraphics[width=\imgw]{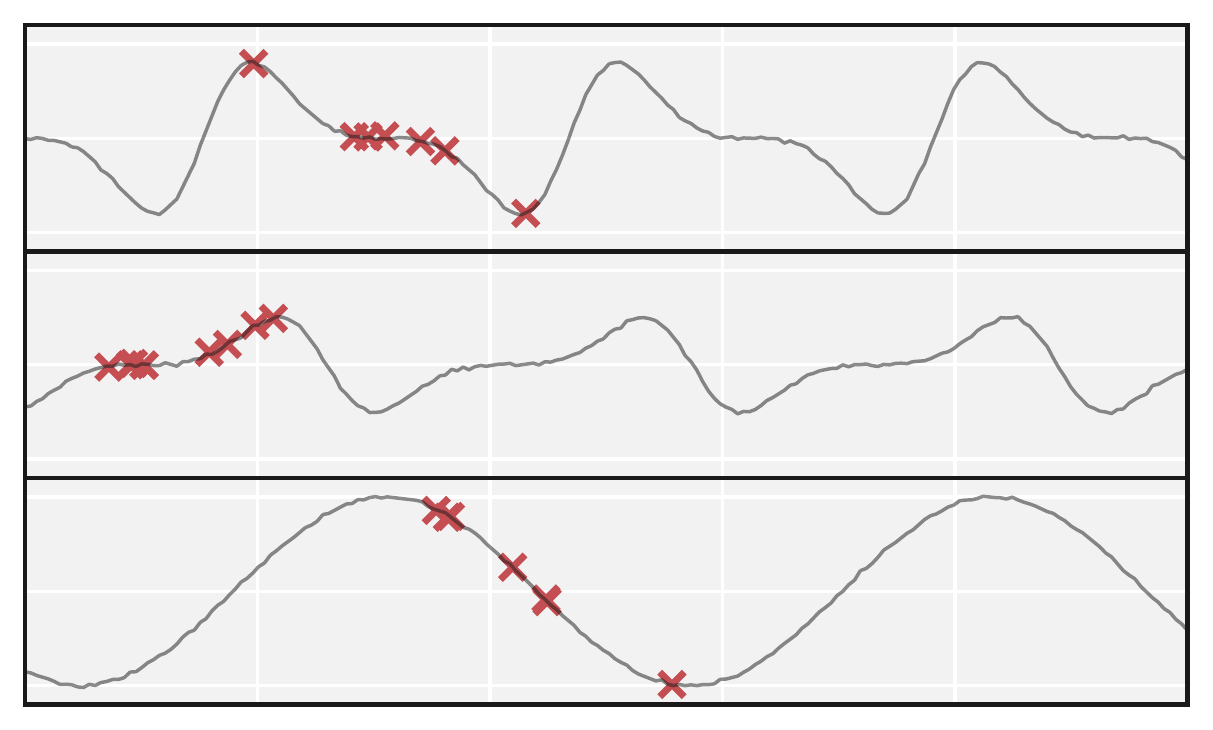}
    };
    \node[inner sep=0pt,right=\hgap of r12,
    ] (r13) {
      \includegraphics[width=\imgw]{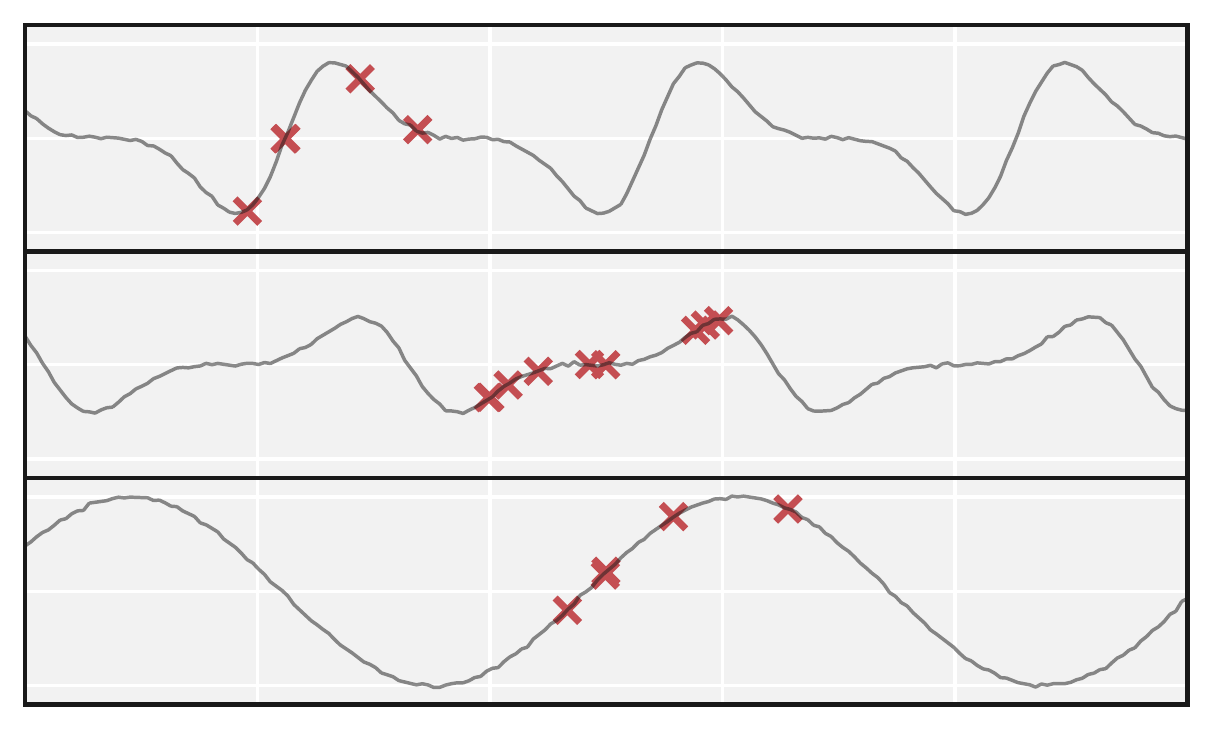}
    };
    \node[inner sep=0pt,right=\hgap of r13,
    ] (r14) {
      \includegraphics[width=\imgw]{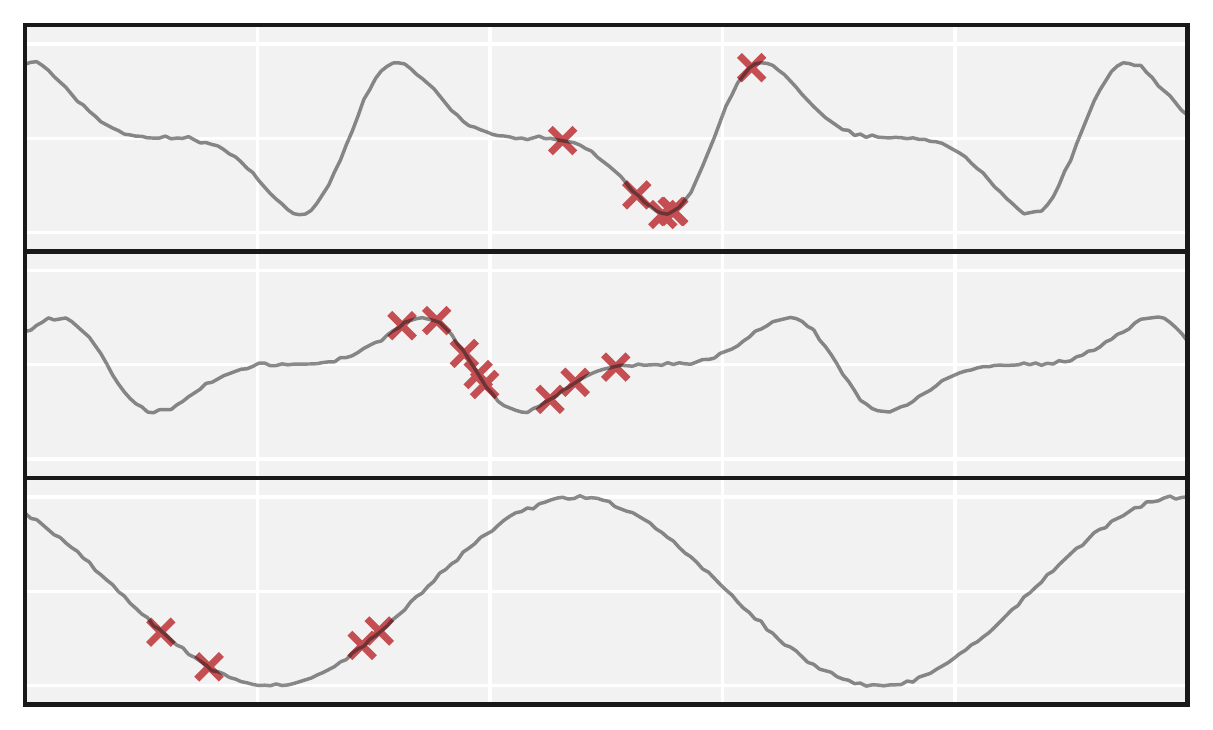}
    };
    \node[inner sep=0pt,below=\vgap of r11,
      label={[rotate=90,anchor=south]left:{\scriptsize\textsf{Cont P-VAE}}}
    ] (r21) {
      \includegraphics[width=\imgw]{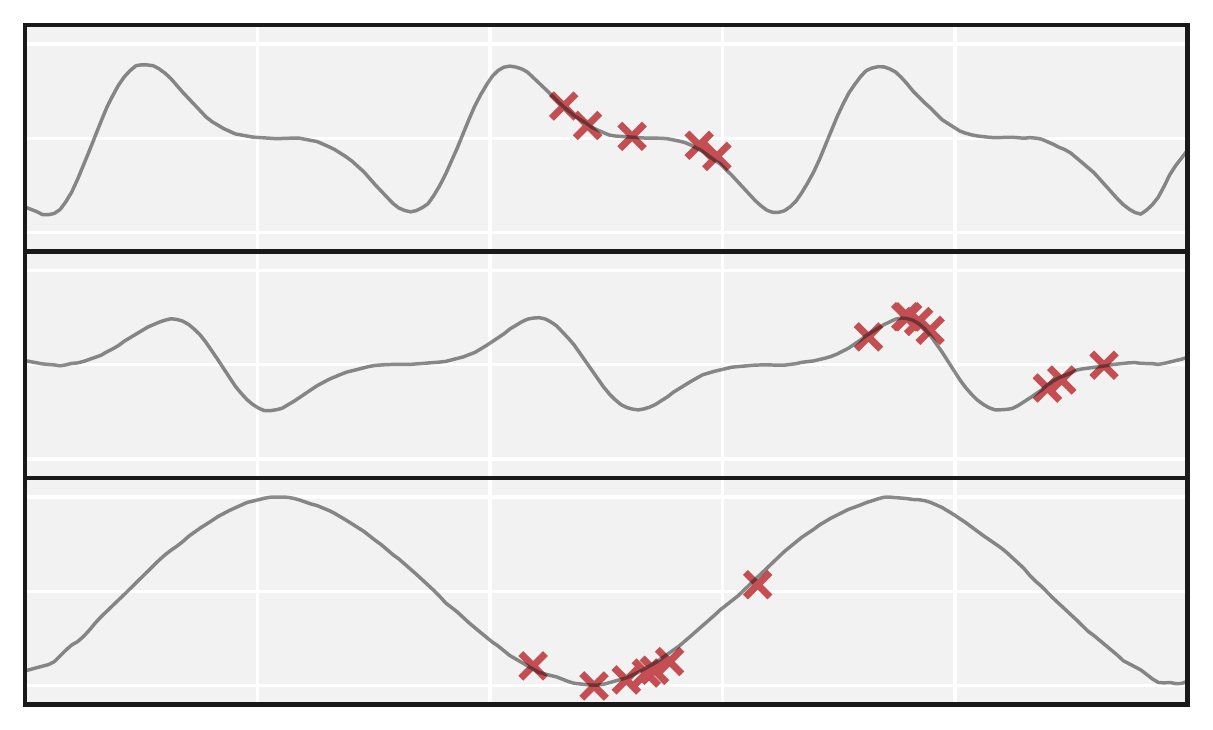}
    };
    \node[inner sep=0pt,right=\hgap of r21,
    ] (r22) {
      \includegraphics[width=\imgw]{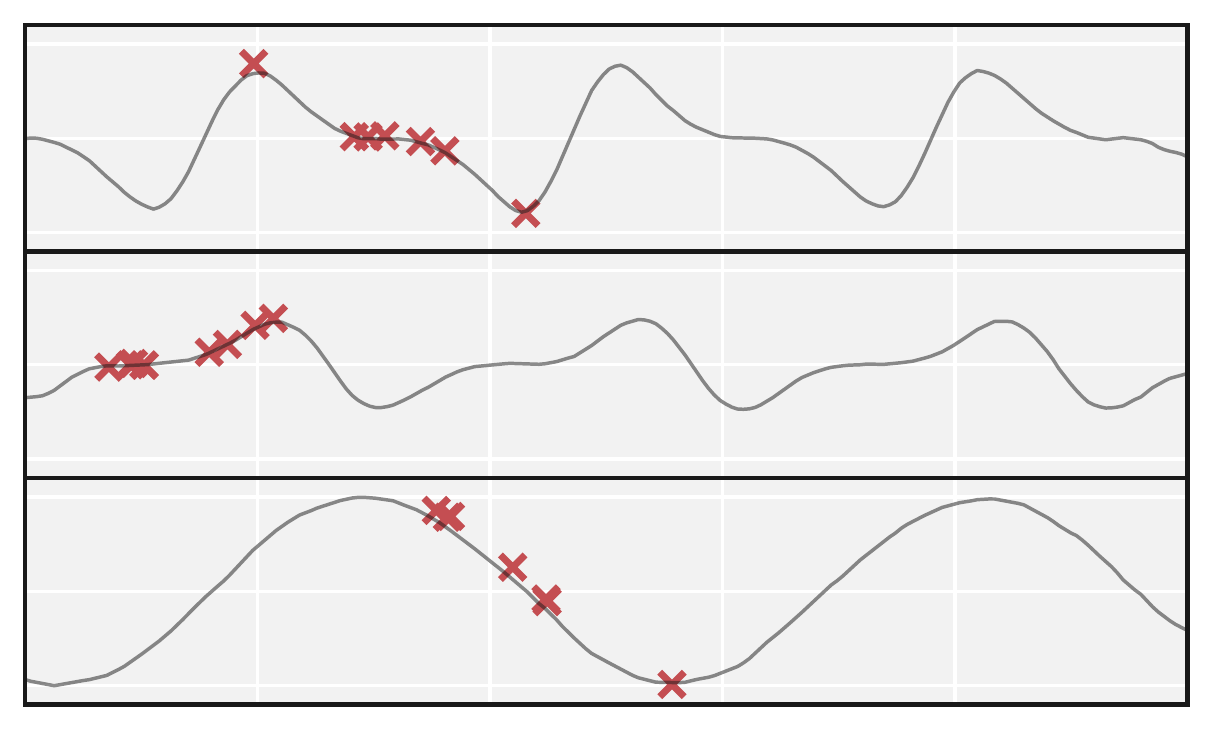}
    };
    \node[inner sep=0pt,right=\hgap of r22,
    ] (r23) {
      \includegraphics[width=\imgw]{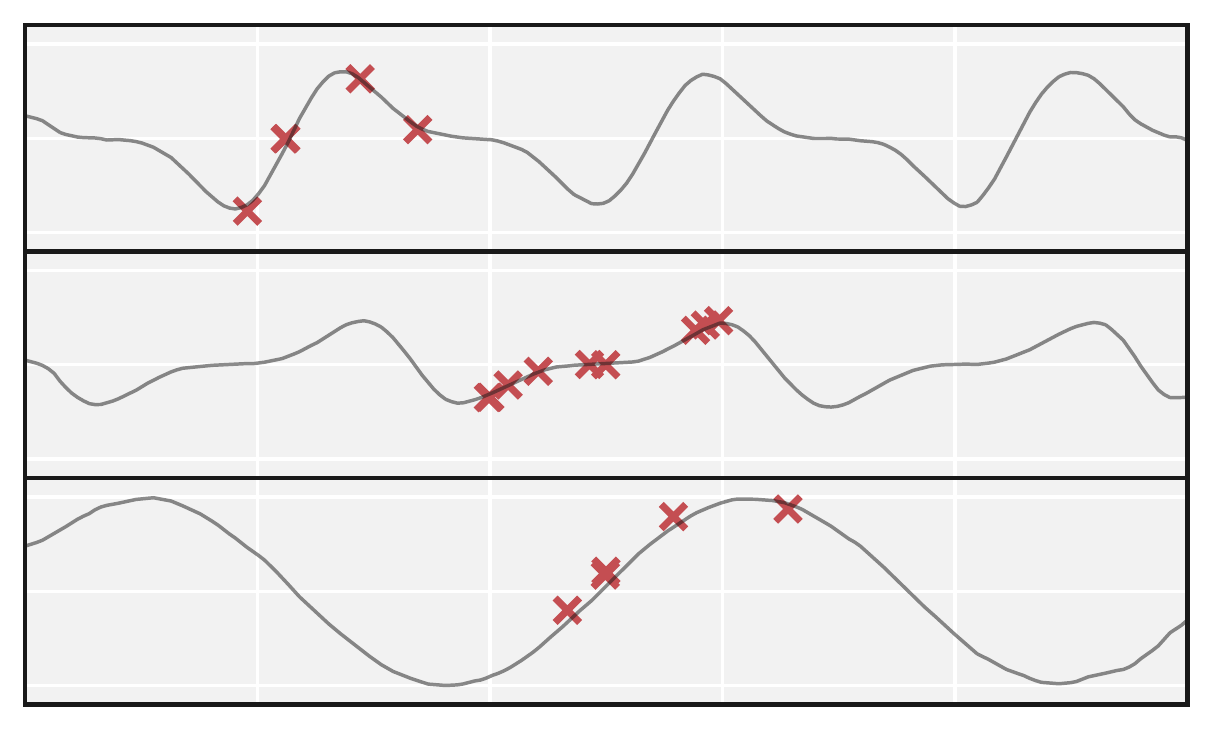}
    };
    \node[inner sep=0pt,right=\hgap of r23,
    ] (r24) {
      \includegraphics[width=\imgw]{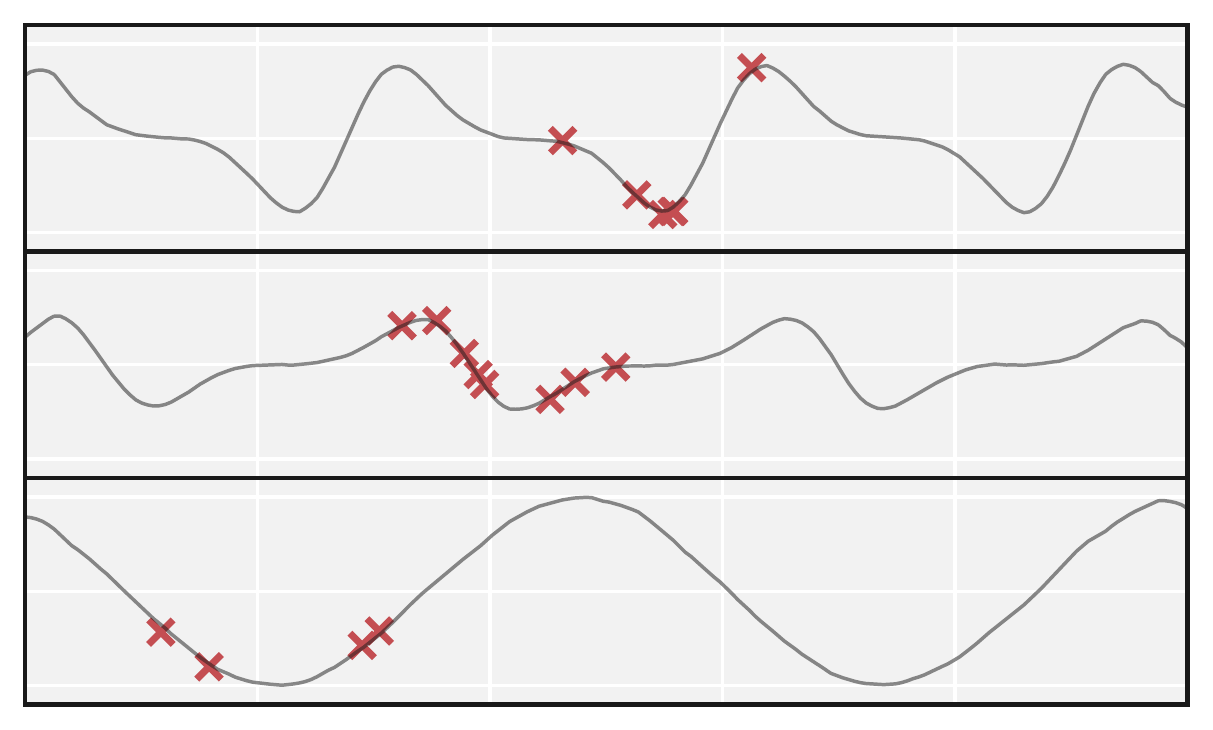}
    };
    \node[inner sep=0pt,below=\vgap of r21,
      label={[rotate=90,anchor=south]left:{\scriptsize\textsf{Cont P-BiGAN}}},
      label={below:{\footnotesize{(a)}}}
    ] (r31) {
      \includegraphics[width=\imgw]{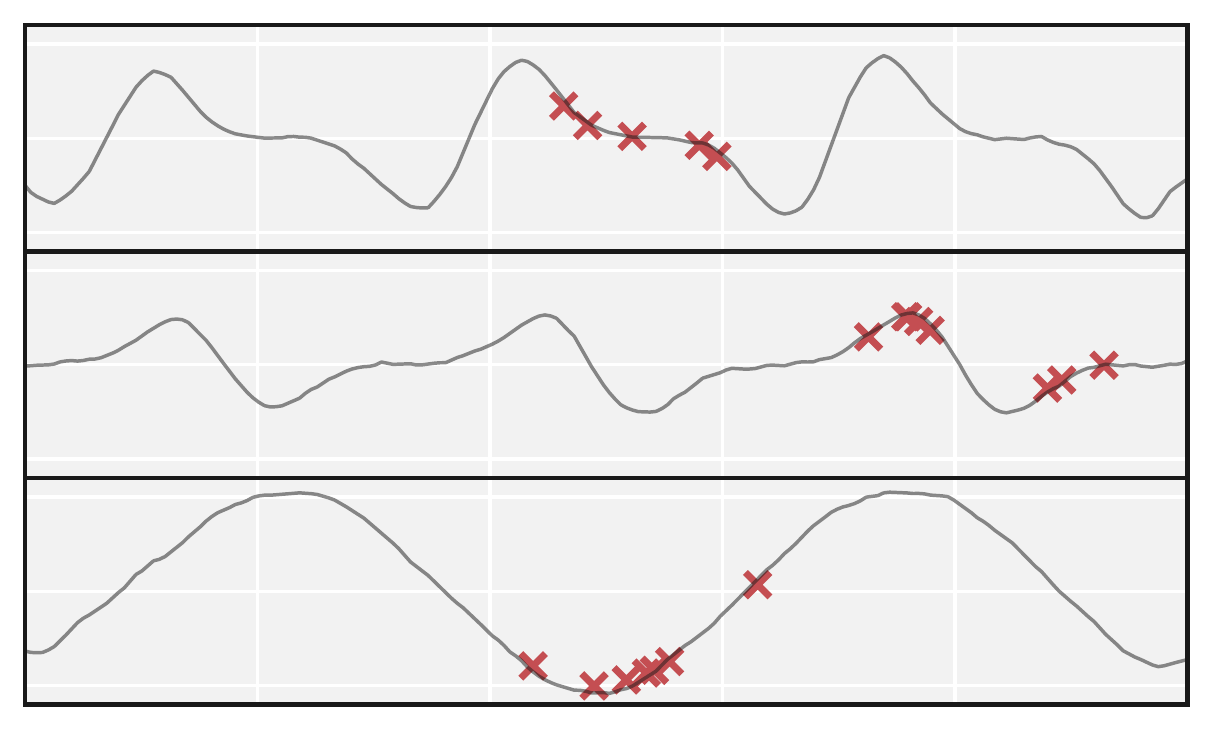}
    };
    \node[inner sep=0pt,right=\hgap of r31,
      label={below:{\footnotesize{(b)}}}
    ] (r32) {
      \includegraphics[width=\imgw]{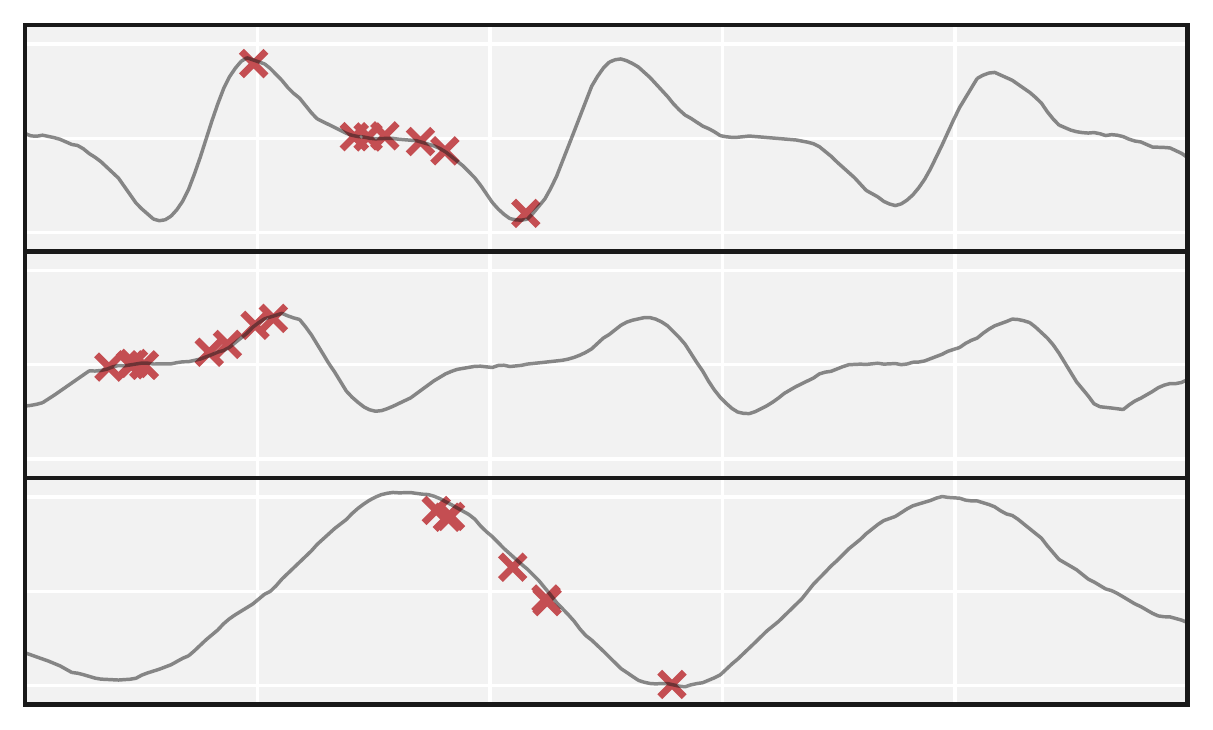}
    };
    \node[inner sep=0pt,right=\hgap of r32,
      label={below:{\footnotesize{(c)}}}
    ] (r33) {
      \includegraphics[width=\imgw]{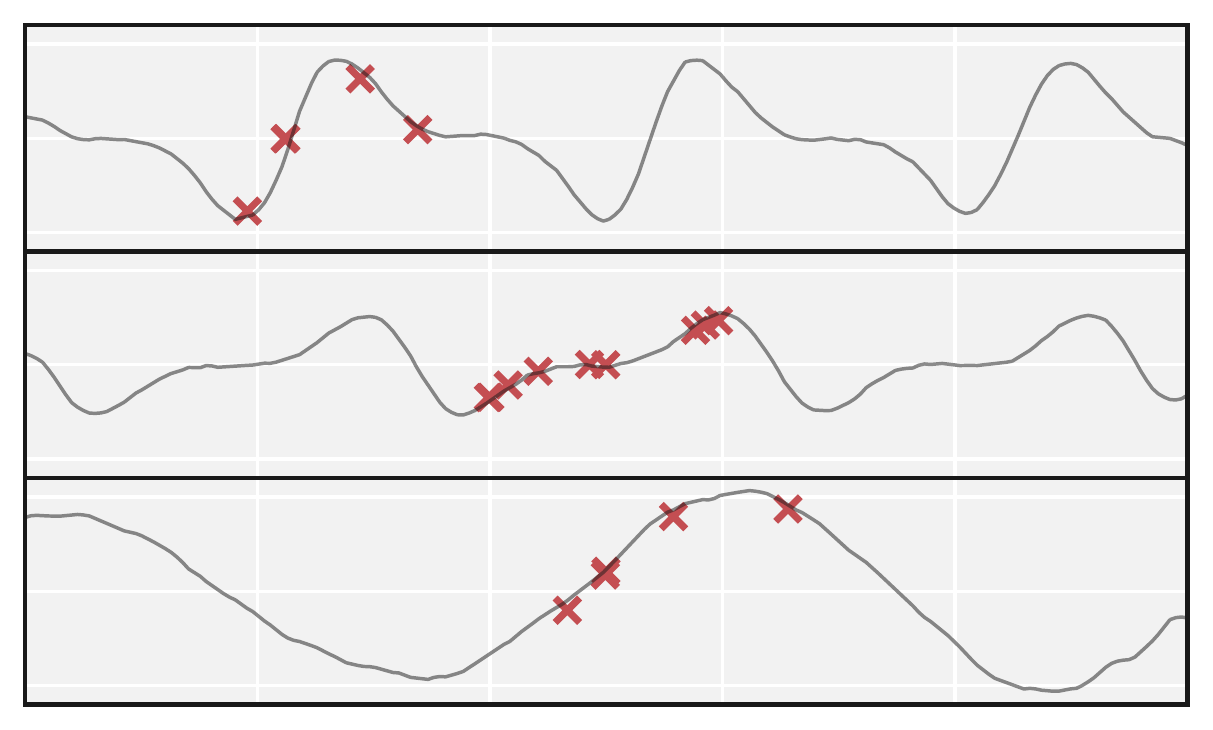}
    };
    \node[inner sep=0pt,right=\hgap of r33,
      label={below:{\footnotesize{(d)}}}
    ] (r34) {
      \includegraphics[width=\imgw]{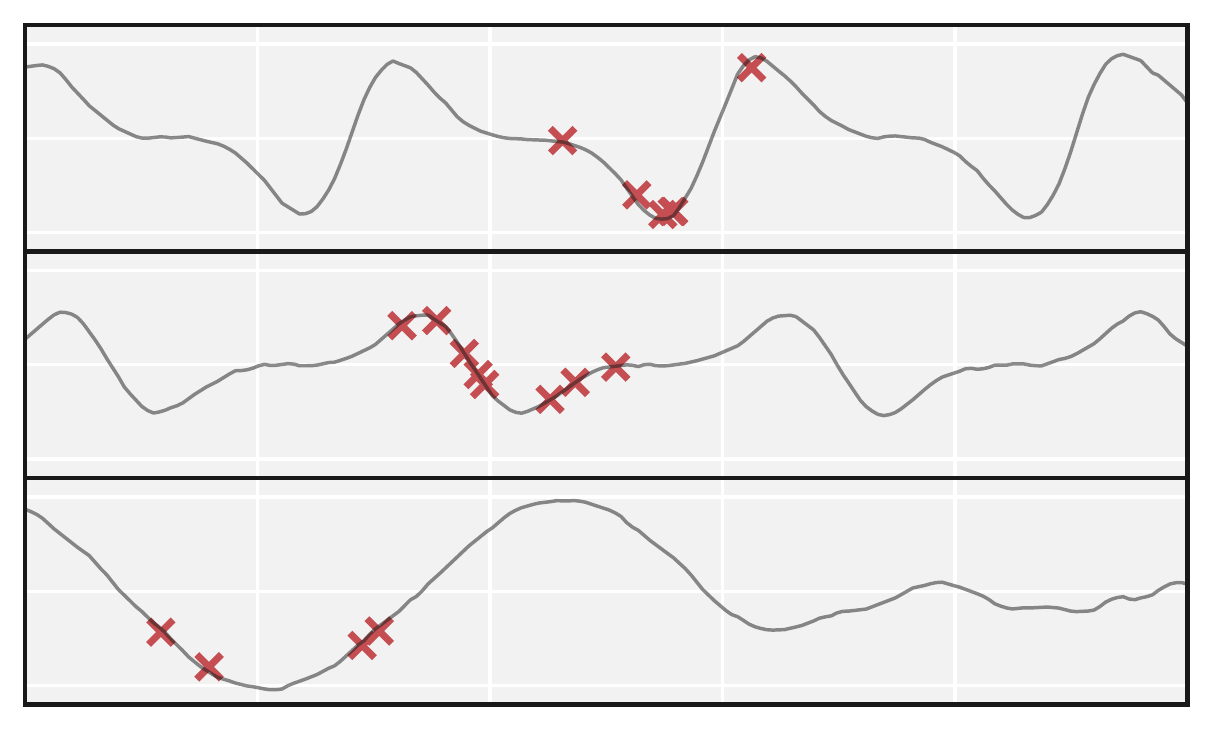}
    };
  \end{tikzpicture}
  \caption{
    Imputation results of Cont P-VAE and Cont P-BiGAN on
    a 3-channel synthetic time series.
    The first row shows four random samples from the training data.
    Each sample has three channels displayed as a group and
    the observations in each channel are shown as the red markers,
    which are drawn from the latent temporal function plotted as
    the gray trajectory.
    The second and the third rows show the inferred latent trajectory
    of each channel,
    conditioned on the same observations shown in the first row
    by Cont P-VAE and Cont P-BiGAN respectively.
    We can see that in general Cont P-VAE produces visually better
    completion results that are consistent with the overall structure
    of the training samples.
    On the other hand, the inferred trajectories of P-BiGAN are less
    smooth (zoom-in to see the details), and it seems that P-BiGAN
    captures more easily the Gaussian noise added in the training data.
    However, P-BiGAN generally produces relatively poor imputation results
    that do not have the consistent overall structure such as
    the right tail in channel 3 of case (c)
    and the right tail in channel 3 of case (d).
    This is similar to the case of high missing rate with
    independent dropout missingness in
    Section~\ref{sec:impute fid},
    as the time series are very sparsely observed
    (7.4 observations in each channel on average).
    Note that if we trained both model on a more densely sampled time series,
    such as the one with times drawn from a homogeneous Poisson process
    with rate $\lambda=200$, the two models will behave similarly.
  }
  \label{fig:toy}
\end{figure*}

\begin{figure*}
  \centering
  \def\vgap{1em}
  \def\hgap{0em}
  \def\toydir{synthetic}
  \def\imgw{.24\textwidth}
  \begin{tikzpicture}
    \node[inner sep=0pt,
      label={[rotate=90,anchor=south]left:{\scriptsize\textsf{Cont P-VAE}}}
    ] (r11) {
      \includegraphics[width=\imgw]{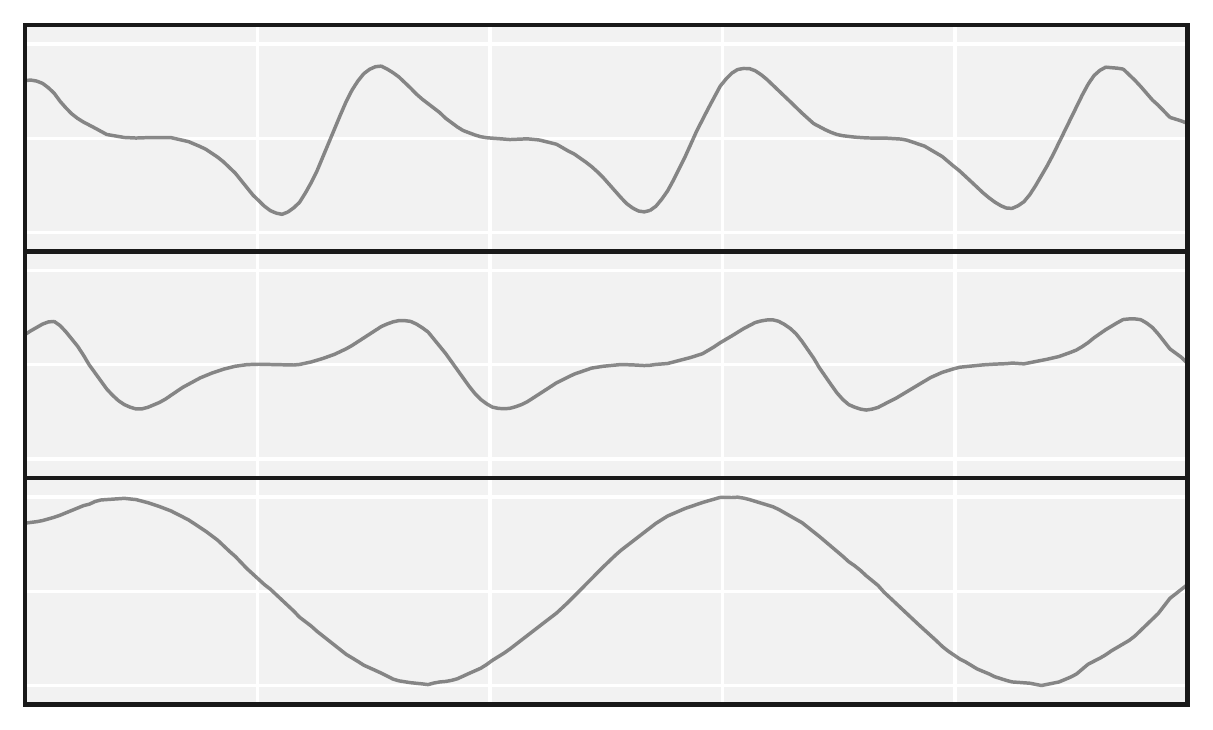}
    };
    \node[inner sep=0pt,right=\hgap of r11,
    ] (r12) {
      \includegraphics[width=\imgw]{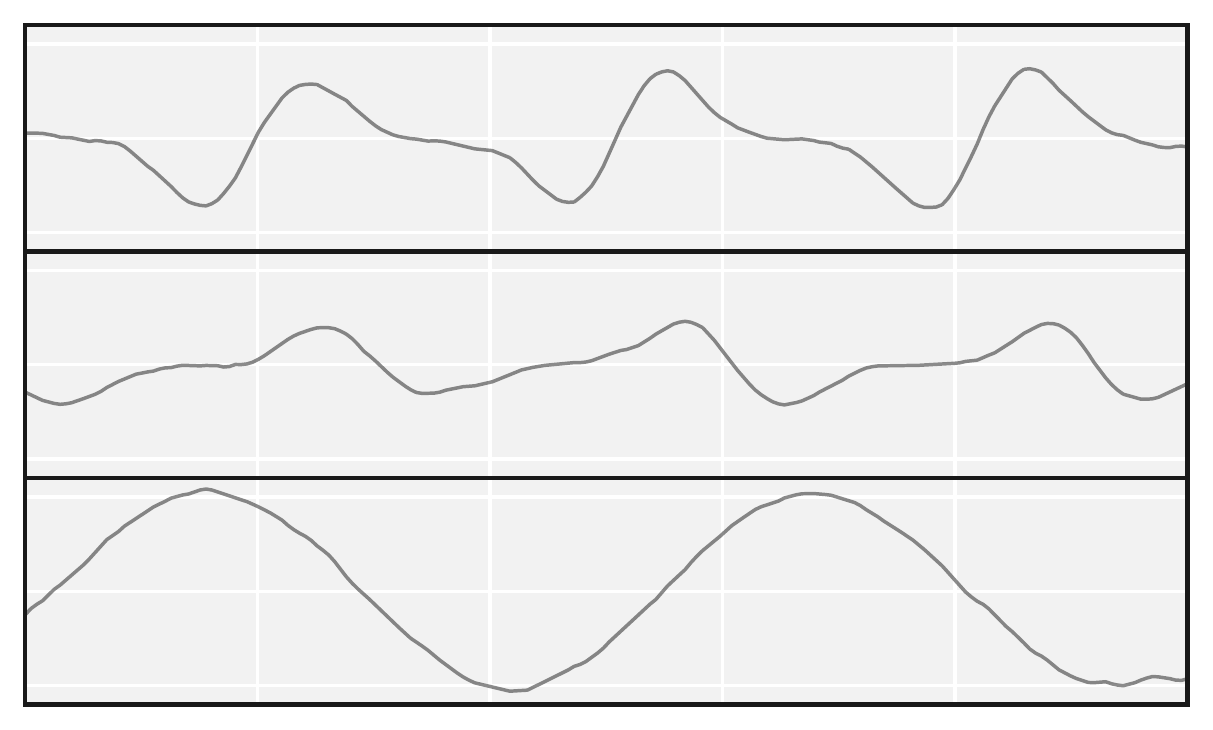}
    };
    \node[inner sep=0pt,right=\hgap of r12,
    ] (r13) {
      \includegraphics[width=\imgw]{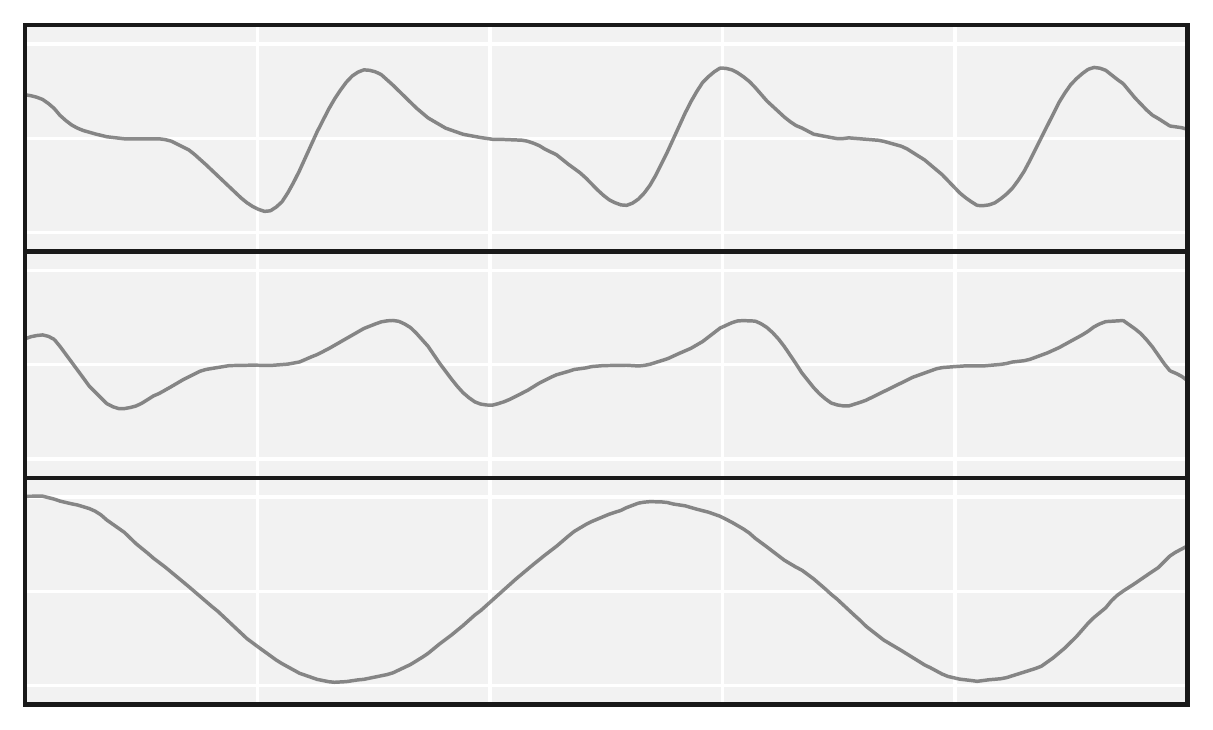}
    };
    \node[inner sep=0pt,right=\hgap of r13,
    ] (r14) {
      \includegraphics[width=\imgw]{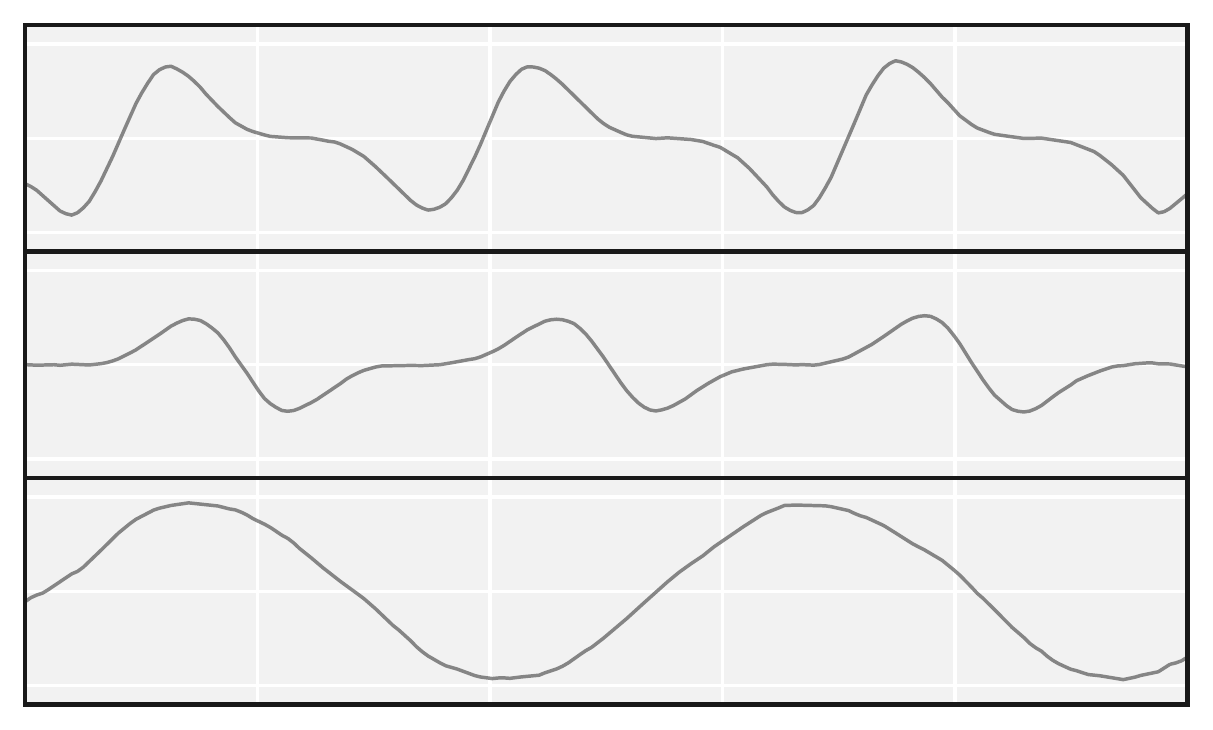}
    };
    \node[inner sep=0pt,below=\vgap of r11,
      label={[rotate=90,anchor=south]left:{\scriptsize\textsf{Cont P-BiGAN}}},
      label={below:{\footnotesize{(a)}}}
    ] (r21) {
      \includegraphics[width=\imgw]{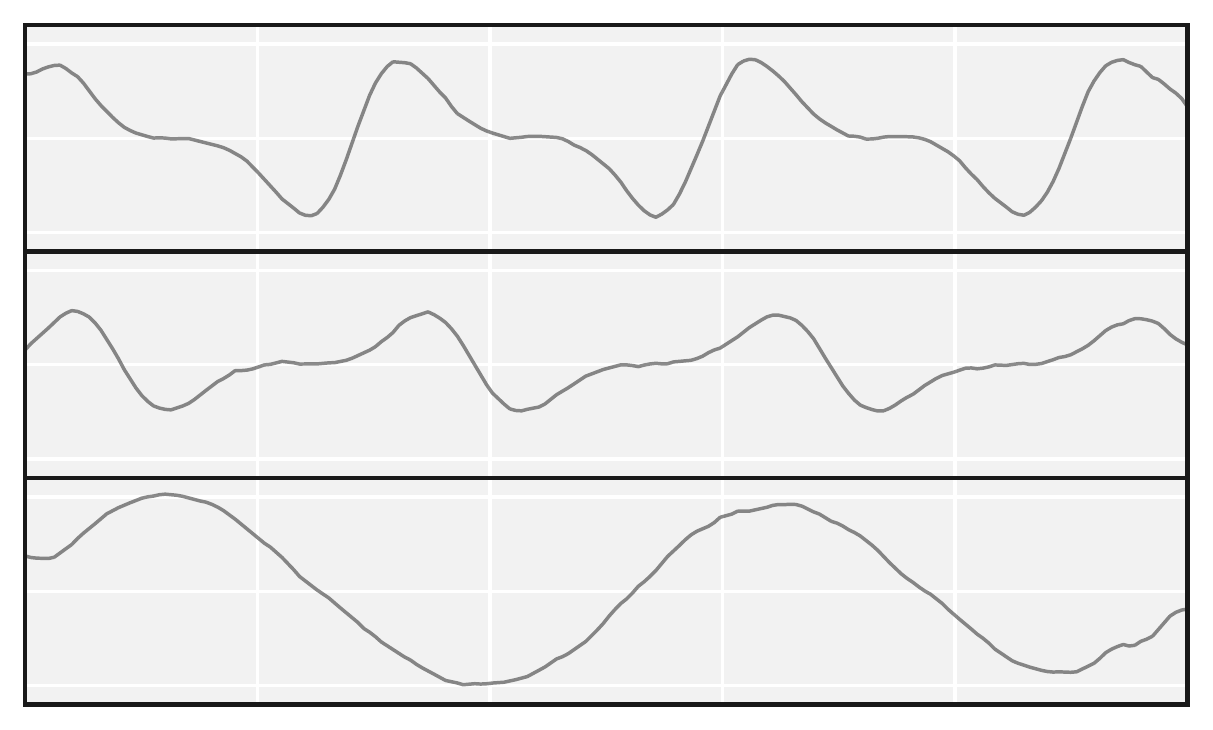}
    };
    \node[inner sep=0pt,right=\hgap of r21,
      label={below:{\footnotesize{(b)}}}
    ] (r22) {
      \includegraphics[width=\imgw]{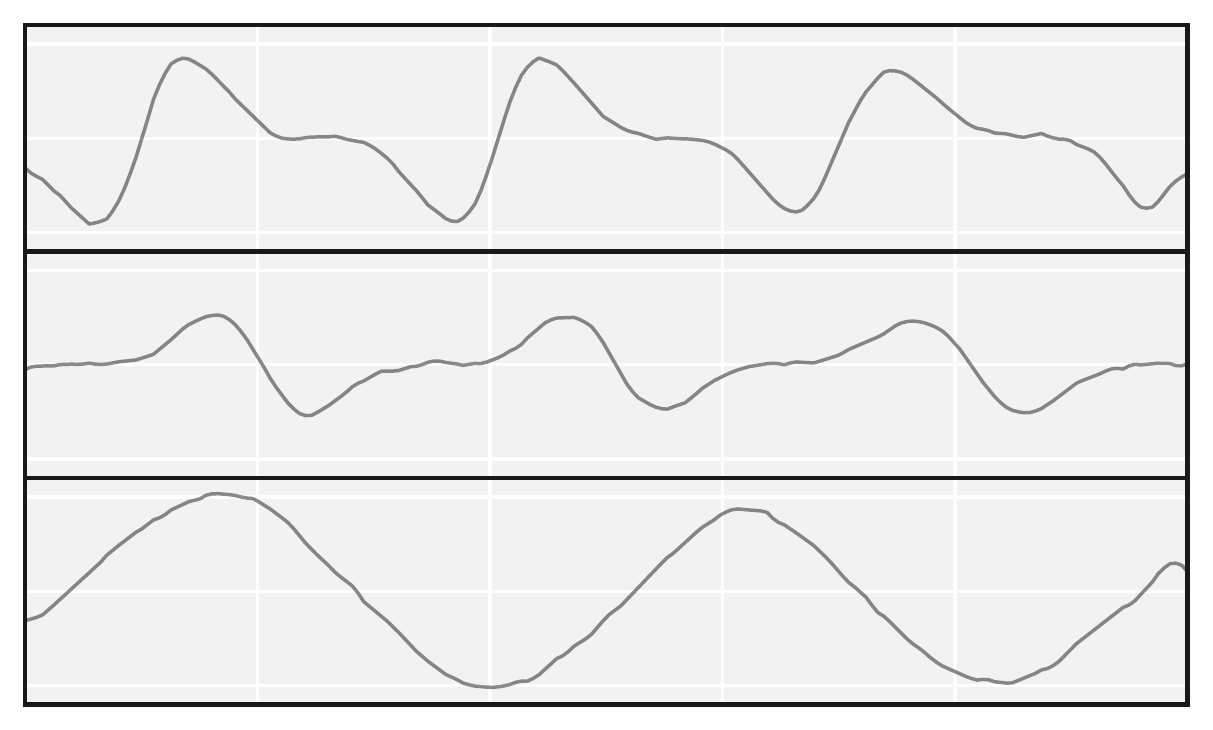}
    };
    \node[inner sep=0pt,right=\hgap of r22,
      label={below:{\footnotesize{(c)}}}
    ] (r23) {
      \includegraphics[width=\imgw]{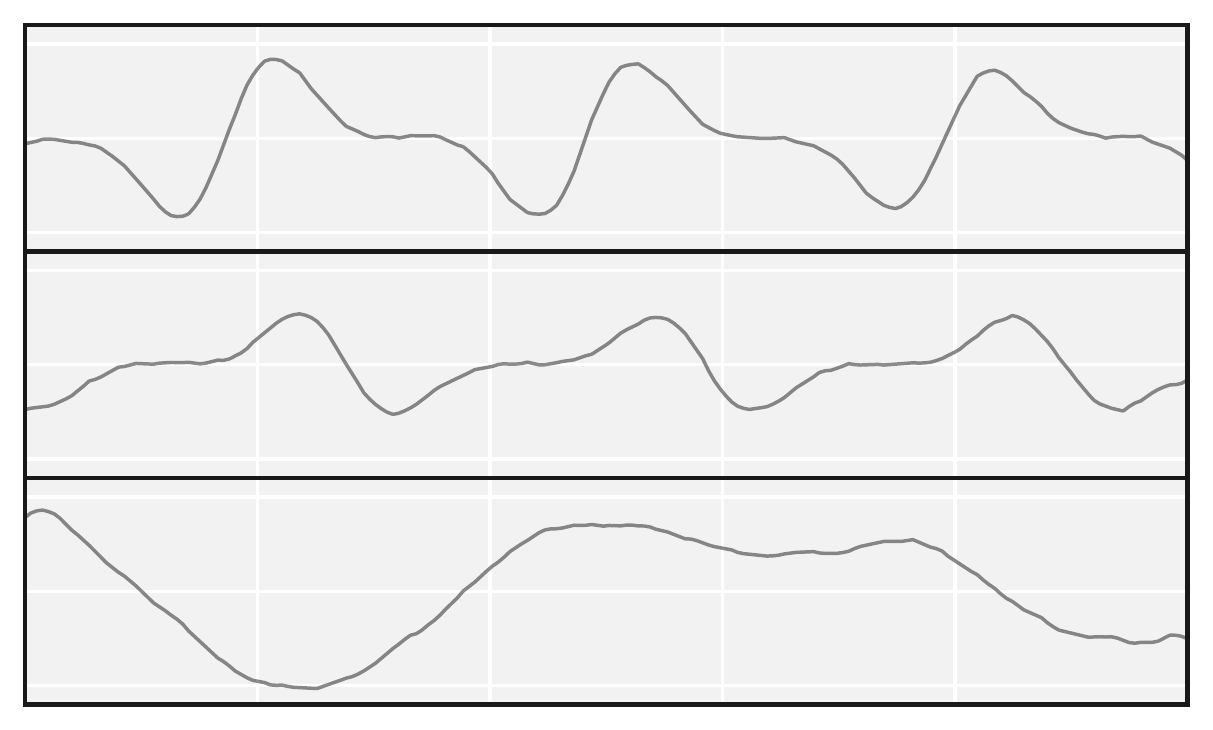}
    };
    \node[inner sep=0pt,right=\hgap of r23,
      label={below:{\footnotesize{(d)}}}
    ] (r24) {
      \includegraphics[width=\imgw]{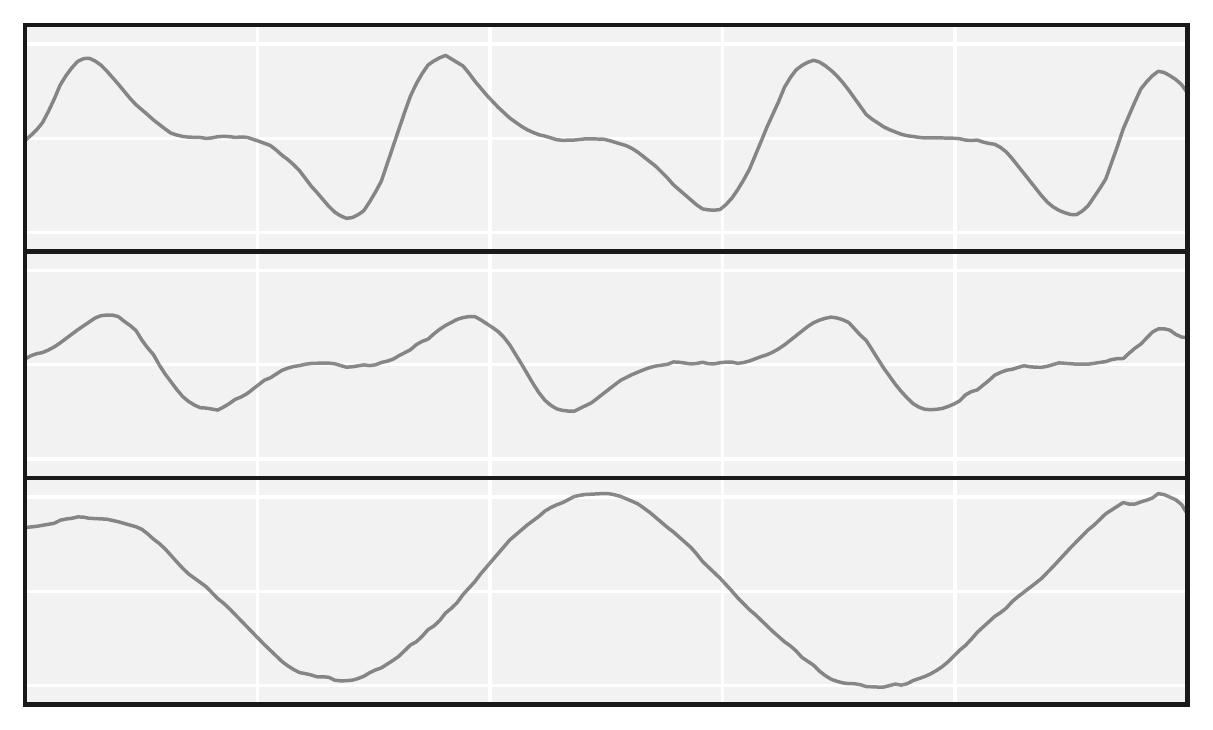}
    };
  \end{tikzpicture}
  \caption{
    Randomly generated samples
    by Cont P-VAE (first row) and Cont P-BiGAN (second row) trained on the
    synthetic time series shown in Figure~\ref{fig:toy}.
    Similar to the imputation results,
    Cont P-VAE produces smoother trajectories
    that are consistent with the ground truth generative process.
    On the contrary, occasionally there are artifacts in the samples
    generated by Cont P-BiGAN such as the trajectory of the third channel
    in case (c).
  }
  \label{fig:toy gen}
\end{figure*}


In this section,
we equip P-VAE and P-BiGAN with the continuous decoder and encoder
described in Section~\ref{sec:cont time} and demonstrate how they work
on a synthetic time series dataset using the same architecture
described in Section~\ref{sec:mimic3}.
We generate a dataset containing 10,000 time series each with three channels over $t\in[0, 1]$
according to the following generative process:
\begin{align*}
  a &\sim \N(0, 10^2) \\
  b &\sim \operatorname{uniform}(0, 10) \\
  f_1(t) &= .8\sin(20(t + a) + \sin(20(t + a))) \\
  f_2(t) &= -.5\sin(20(t + a + 20) + \sin(20(t + a + 20))) \\
  f_3(t) &= \sin(12(t + b))
\end{align*}
where an independent Gaussian noise $\N(0, 0.01^2)$ is added
to each channel.

The observation time points for each channel are drawn independently
from a homogeneous Poisson process with rate $\lambda=30$
sampled continuously within $[d, d+0.25]$
where $d\sim\text{uniform}(0,0.75)$.
This results in 7.4 observations in each channel on average.
The first row of Figure~\ref{fig:toy} shows some examples from
the generated synthetic dataset.


Figure~\ref{fig:toy} and \ref{fig:toy gen}
shows that both P-VAE and P-BiGAN are able to
learn the generative distribution reasonably
given the sparsely and irregularly-sampled observations.
They are both able to learn the periodic dynamics
and infer the latent functions according to sparse observations.
Moreover, both models also learn that the first two channels are correlated
due to the shared random offset $a$ in the generative process,
and the shifting of the third channel is uncorrelated to the first two
channels as shown in Figure~\ref{fig:toy gen}.

From the plots, we can see that P-VAE tends to generate smoother
curves, while P-BiGAN captures the detailed fluctuation caused by the
added Gaussian noise.
This is similar to the results on image modeling
shown in Section~\ref{sec:impute fid}:
GAN-based models capture the local details better
but the results can be noisy when the spatial signals are weak.
On the contrary,
VAE-based models learn the big picture better but the results are
usually smoother.

\section{Details of Experiments}

\subsection{Data Preparation and Preprocessing}
MNIST can be downloaded from: \\
{\url{http://yann.lecun.com/exdb/mnist/}}

CelebA can be downloaded from: \\
{\url{http://mmlab.ie.cuhk.edu.hk/projects/CelebA.html}}

For both MNIST and CelebA, the range of pixel values of the image
is rescaled to $[0, 1]$.


MIMIC-III can be downloaded following the instructions from
its website: \\
{\url{https://mimic.physionet.org/gettingstarted/access/}}

We follow the GitHub repository below to preprocess the
MIMIC-III dataset: \\
{\url{https://github.com/mlds-lab/interp-net}}

For MIMIC-III, we normalize the timestamps within 48 hours
to the interval $[0, 1]$.
The observed values of the time series are rescaled
to $[-1, 1]$ according to the minimum and maximum value of each channel
across the entire training set.

\subsection{Reference Implementations}

We use the following reference implementation for the baseline models
in our experiments.

MisGAN: \\
{\url{https://github.com/steveli/misgan}}

GRU-D: \\
{\url{https://github.com/fteufel/PyTorch-GRU-D}}

Latent ODE: \\
{\url{https://github.com/YuliaRubanova/latent_ode}}

M-RNN: \\
{\url{https://github.com/jsyoon0823/MRNN}}

The continuous convolutional layer described in Section~\ref{sec:cont conv}
is built upon the spline-based convolution operator: \\
{\url{https://github.com/rusty1s/pytorch_spline_conv}}

\subsection{Hyperparameters}
\label{sec:hyperparams}

Most of the hyperparameters of our models used in the experiments
are manually chosen as described in Section~\ref{sec:experiments}
without further tuning and
are specified in the provided implementation.
The only hyperparameter we tune is
the strength of the autoencoding loss of P-BiGAN,
the coefficient $\lambda$ in objective \eqref{eq:p-bigan-ae},
for the CelebA experiments.
We vary $\lambda$ from $\{0, 10^{-5}, 10^{-4}, 10^{-3}, 10^{-2}, 10^{-1}\}$
and choose the one that yields the best FID.
We found that tuning this hyperparameter makes a significant difference
for different missing patterns.
For block observation, smaller $\lambda$ yields better results;
while for independent dropout, larger $\lambda$ yields better results.

\subsection{Computing Infrastructure}
All of our experiments are computed using the NVIDIA GeForce GTX 1080 Ti GPUs.

\begin{figure}
  \centering
  \scalebox{.85}{
  \begin{tikzpicture}[scale=1.4, >=latex']
    \tikzstyle{every node} = [circle, minimum size=20pt, inner sep=0pt]
    \newcommand\xdis{.8}
    \newcommand\obscolor{gray!25}
    \newcommand\arrow{-Stealth[length=1.6mm,width=1.6mm]}
    \node[draw, thick, fill=\obscolor] (x) at (0, 1) {$\xx$};
    \node[draw, thick, fill=\obscolor] (t) at (\xdis, 1) {$\tt$};
    \node[draw, thick] (z) at (0, -1) {$\zz$};
    \node[draw, thick] (xhat) at (0, -3) {$\widehat{\xx}$};
    \node [trapezium, trapezium angle=-60, minimum width=10, draw, thick]
    (q) at (0, 0) {$q_\phi$};
    \node [trapezium, trapezium angle=60, minimum width=10, draw, thick]
    (g) at (0, -2) {$g_\theta$};
    \coordinate (dummy) at (\xdis, -.8);
    \draw[\arrow] (x) -- (q);
    \draw[\arrow] (t) to[out=270, in=90] (q);
    \draw[\arrow] (q) -- (z);
    \draw[-] (t) -- (dummy);
    \draw[\arrow] (z) -- (g);
    \draw[\arrow] (g) -- (xhat);
    \draw[\arrow] (dummy) to[out=270, in=90] (g);
    \node[rectangle] at ($(x)!.5!(t) + (0, .6)$) {$(\xx,\tt)\sim p_\D$};
  \end{tikzpicture}
  }
  \caption{The structure of P-VAE.
  $q_\phi$ is the encoder and $g_\theta$ is the decoder.}
\vspace{3em}
  \centering
  \scalebox{.85}{
  \begin{tikzpicture}[scale=1.4, >=latex']
    \tikzstyle{every node} = [circle, minimum size=20pt, inner sep=0pt]
    \newcommand\aecolor{brown}
    \newcommand\obscolor{gray!25}
    \newcommand\arrow{-Stealth[length=1.6mm,width=1.6mm]}
    \begin{scope}
      \newcommand\xdis{.8}
      \node[draw, thick, fill=\obscolor] (x) at (0, 1) {$\xx$};
      \node[draw, thick, fill=\obscolor] (t) at (\xdis, 1) {$\tt$};
      \node[draw, thick] (z) at (0, -1) {$\zz$};
      \node[draw, thick, color=\aecolor] (xhat) at (0, -3) {$\widehat{\xx}$};
      \node [trapezium, trapezium angle=-60, minimum width=10, draw, thick]
      (q) at (0, 0) {$q_\phi$};
      \node [trapezium, trapezium angle=60, minimum width=10, draw, thick,
             color=\aecolor]
      (g) at (0, -2) {$g_\theta$};
      \coordinate (dummy) at (\xdis, -.8);
      \draw[\arrow] (x) -- (q);
      \draw[\arrow] (q) -- (z);
      \draw[-, color=\aecolor] (t) -- (dummy);
      \draw[\arrow] (t) to[out=270, in=90] (q);
      \draw[\arrow, color=\aecolor] (z) -- (g);
      \draw[\arrow, color=\aecolor] (g) -- (xhat);
      \draw[\arrow, color=\aecolor] (dummy) to[out=270, in=90] (g);
      \node[draw, fit={(-\xdis/2-.5, -1.5) (\xdis+.5, 1.5)}, rectangle,
      rounded corners, dashed, inner sep=0pt] (box1) {};
    \end{scope}
    \begin{scope}[xshift=6em]
      \newcommand\xdis{.8}
      \node[draw, thick] (x') at (\xdis/2, 1) {$\xx'$};
      \node[draw, thick, fill=\obscolor] (t') at (\xdis, -1) {$\tt'$};
      \node[draw, thick, fill=\obscolor] (z') at (0, -1) {$\zz'$};
      \node [trapezium, trapezium angle=-60, minimum width=10, draw, thick]
      (g') at (\xdis/2, 0) {$g_\theta$};
      \draw[\arrow] (g') -- (x');
      \draw[\arrow] (t') to[out=90, in=270] (g');
      \draw[\arrow] (z') to[out=90, in=270] (g');
      \node[draw, fit={(-0.5, -1.5) (\xdis+.5, 1.5)}, rectangle,
      rounded corners, dashed, inner sep=0pt] (box2) {};
    \end{scope}
    \node [draw, thick, rectangle, inner sep=7pt]
    (d) at ($(box1)!.5!(box2) + (0,2.1)$) {$D$};
    \path [draw,\arrow,rounded corners=5pt] (box1.north) |-
    node[above, rectangle] {$\{(\xx,\tt,\zz)\}$} (d.west);
    \path [draw,\arrow,rounded corners=5pt] (box2.north) |-
    node[above, rectangle] {$\{(\xx',\tt',\zz')\}$} (d.east);

    \node[rotate=90, color=\aecolor] at ($(x)!.5!(xhat) + (-1.5, 0)$)
    {$\ell(\xx,\widehat{\xx})$};
    \path [draw,rounded corners=5pt, color=\aecolor]
    (x.west) -- ++(-1,0) |- (xhat.west);

    \node[rectangle] at ($(z')!.5!(t') + (-.1, -1.5)$) {
      \begin{tabular}{c}
        sampling process:\\[.5em]
        $\begin{aligned}
          (\xx,\tt) &\sim p_\D \\
          (\cdot,\tt') &\sim p_\D \\
          \zz' &\sim p_z
        \end{aligned}$
      \end{tabular}
    };
  \end{tikzpicture}
  }
  \caption{P-BiGAN with autoencoding regularization.
    $q_\phi$ is the stochastic encoder.
    $g_\theta$ is the deterministic decoder;
    the two $g_\theta$ share the same parameters.
    $D$ is the discriminator that takes as input a collection of tuples
    $(\xx,\tt,\zz)$ and $(\xx',\tt',\zz')$.
    $\ell(\xx,\widehat{\xx})$ is the autoencoding loss.
    $p_\D$ denotes the empirical distribution of the training dataset $\D$
    and $p_z$ is the prior distribution of the latent code $\zz$.
    The part in brown is for additional autoencoding regularization.
  }
  \label{fig:pbigan ae}
\end{figure}

\end{document}